\documentclass{article}

\usepackage[preprint]{neurips_2026}

\usepackage[utf8]{inputenc} 
\usepackage[T1]{fontenc}    
\usepackage{hyperref}       
\usepackage{url}            
\usepackage{booktabs}       
\usepackage{amsfonts}       
\usepackage{nicefrac}       
\usepackage{microtype}      
\usepackage{xcolor}         

\title{How Neural Reward Models Learn Features for Policy Optimization: A Single-Index Analysis}

\author{%
  \textbf{Rei Higuchi\textsuperscript{1,2}}\thanks{Correspondence to: \texttt{higuchi-rei714@g.ecc.u-tokyo.ac.jp}.} \quad
  \textbf{Ryotaro Kawata\textsuperscript{1,2}} \quad
  \textbf{Akifumi Wachi\textsuperscript{3}}
  \\
  \textbf{Shokichi Takakura\textsuperscript{3}} \quad
  \textbf{Kohei Miyaguchi\textsuperscript{3}} \quad
  \textbf{Taiji Suzuki\textsuperscript{1,2}}
  \\
  \textsuperscript{1}The University of Tokyo \quad
  \textsuperscript{2}RIKEN AIP \quad
  \textsuperscript{3}LY Corporation
}

\usepackage{amsmath}
\usepackage{amssymb}
\usepackage{mathtools}
\usepackage{amsthm}

\usepackage{enumitem}

\usepackage[capitalize,noabbrev]{cleveref}

\theoremstyle{plain}
\newtheorem{theorem}{Theorem}[section]
\newtheorem{proposition}[theorem]{Proposition}
\newtheorem{lemma}[theorem]{Lemma}
\newtheorem{corollary}[theorem]{Corollary}
\theoremstyle{definition}
\newtheorem{definition}[theorem]{Definition}

\theoremstyle{remark}
\newtheorem{remark}[theorem]{Remark}

\usepackage[textsize=tiny]{todonotes}

\definecolor{editcoolcolor}{RGB}{0,155,255}  

\definecolor{editwarmcolor}{RGB}{255,125,0}  

\definecolor{editgreencolor}{RGB}{0,150,70}  

\definecolor{editmagentacolor}{RGB}{190,0,190}  

\begin{document}

\maketitle

\begin{abstract}
Reward modeling is not only a prediction problem: in KL-regularized policy
optimization, the learned reward is exponentiated to define the deployed
policy, so downstream value depends on errors in reward-tilted regions. We
study this feedback in a Gaussian single-index model with
$r^*(x)=\sigma^*(\langle\theta^*,x\rangle)$ and
$x\sim\mathcal N(0,I_d)$.
We analyze a two-stage neural reward model that first learns the hidden
direction $\theta^*$ from reward-weighted samples and then fits the readout
layer by weighted ridge regression. Exponential reward weighting changes the
Hermite signal available to the first layer; for any feature-learning
temperature $\beta_1$ above a dimension-free $O(1)$ threshold, a constant
fraction of neurons recover the hidden direction, with weak-recovery
complexity governed by the generative exponent.
After feature recovery, we derive tilted-policy value-gap bounds for
an idealized label-weighted fit with weights $e^{y/\beta_2}$ and a more
practical surrogate-weighted fit with weights
$e^{r_{a_0}(x)/\beta_2}$. Keeping the $\beta_2$-dependence explicit
yields an admissible set of deployment temperatures, balancing the gain from
lowering $\beta_2$ against the learning cost amplified by exponential
weighting; in the surrogate-weighted case, proxy-dependent factors shrink
this admissible set.

\end{abstract}

\section{Introduction}

Reward modeling (RM) is a key component of aligning large language models \citep{christiano2017deep,stiennon2020learning,ouyang2022training}, but it differs from ordinary supervised prediction. A learned reward is not merely evaluated on static data; it is optimized to produce a new policy, thereby dictating the distribution on which future rewards are collected and evaluated. In KL-regularized reward maximization, this optimized policy takes the form of an exponential tilt of a reference policy \cite{DBLP:journals/corr/abs-1910-00177,DBLP:conf/nips/KorbakEKD22,DBLP:conf/icml/GoKKRRD23, DBLP:conf/icml/PetersS07}:
\[
    \pi_\beta(dx)
    \propto
    e^{r(x)/\beta}
    \pi_{\mathrm{ref}}(dx),
\]
where $r$ is the reward function and the temperature $\beta>0$ controls how strongly the policy concentrates on high-reward regions. Here $x$ denotes the evaluated input to the reward model; in language-modeling applications, this input is a prompt-response pair. Consequently, reward modeling faces a policy-induced distribution shift.

Existing theoretical research on offline alignment provides insights into this distribution shift, but typically assumes linear reward models over fixed, known features \citep{zhu2023principled,ji2024reinforcement,scheid2024optimal,DBLP:conf/colt/FosterMR25}. While these settings clarify the interaction between coverage and policy optimization, they typically do not address nonlinear reward landscapes or the feature learning problem inherent to training neural networks. To analyze how neural networks discover relevant features and approximate nonlinear rewards under policy-induced distribution shifts, we formulate the problem within a Gaussian single-index model (SIM) , where
\[
    x\sim\pi_{\mathrm{ref}}=\mathcal N(0,I_d),
    \qquad
    r^*(x)=\sigma^*(\langle\theta^*,x\rangle),
    \qquad
    y=r^*(x)+\zeta .
\]
Here $\theta^*$ is the hidden target direction, $\sigma^*$ is the target link function, and $\zeta$ is the observation noise. We focus on upper-bounded polynomial links, which provide a tractable nonlinear reward class while effectively ruling out extreme scenarios of unboundedly large rewards.

In Gaussian SIMs, recovering the hidden direction is governed by Hermite-based exponents. The Information Exponent (IE) captures the first nonzero Hermite signal available to correlation-based methods, while the Generative Exponent (GE) captures the smallest such exponent obtainable after nonlinear transformations of the target link \citep{bietti2022learning,lee2024neural,tsiolis2025information}. Recent work shows that suitable nonlinear transformations can reduce the effective learning bottleneck from IE to GE \citep{lee2024neural,tsiolis2025information,nishikawa2025nonlinear}. 

Our analysis proceeds in two stages: first-layer feature recovery using a feature-learning temperature $\beta_1$, and second-layer nonlinear reward fitting to deploy a policy at a potentially different temperature $\beta_2$. The feature-learning temperature $\beta_1$ controls the stability of the early feature-learning signal, whereas $\beta_2$ governs the downstream tilted policy; the two temperatures need not coincide. For downstream alignment, the relevant metric is not merely prediction risk under the reference distribution, but the value of the policy induced by the learned reward. We study policy-level gaps of the form
$  \mathbb E_{x\sim\pi_{\beta^*}}[r^*(x)]
    -
    \mathbb E_{x\sim\widehat\pi_{\beta_2}}[r^*(x)],$
where $\pi_{\beta^*}$ is a target finite-temperature tilted policy and $\widehat\pi_{\beta_2}$ is the learned tilted policy induced by the fitted reward. Our main contributions are summarized as follows:
\begin{itemize}[leftmargin=*]
    \item \textbf{Feature Learning under Reward-Induced Distribution Shift.}
    We characterize first-layer feature learning under exponential reward weighting in a Gaussian SIM. The weighting acts as a nonlinear label transformation and changes the Hermite signal available for recovering the hidden direction, reducing the effective bottleneck from the IE to the GE. Under upper-bounded polynomial assumptions on the target and student links, bounded noise, and standard Guassian SIM assumptions, a constant fraction of neurons recover the hidden direction for any fixed feature-learning temperature $\beta_1$ above a dimension-free $O(1)$ threshold. This replaces the dimension-growing polylogarithmic-in-$d$ temperature conditions required in prior clipped-exponential analyses \citep{nishikawa2025nonlinear}.

    \item \textbf{From Weighted Prediction to Tilted-Policy Regret.}
    After feature recovery, we evaluate fitted rewards by the true reward obtained under the deployed tilted policy rather than by prediction risk under the reference distribution. The resulting regret bounds separate temperature mismatch, projected-truncation, and learning terms, and explicitly link the policy-value rates to the local shape of the nonlinear reward maximum.

    \item \textbf{Label- and Surrogate-Weighted Regret Bounds.}
    This type of bound was derived for two fitting schemes: an idealized label-weighted estimator with weights $e^{y/\beta_2}$ and a practical surrogate-weighted estimator with weights $e^{r_{a_0}(x)/\beta_2}$ defined by a frozen proxy reward. For both schemes, we prove downstream projected-truncation regret bounds. The surrogate-weighted bound isolates additional proxy-dependent factors that amplify approximation and finite-sample errors when the weighting measure is defined by an imperfect proxy. This identifies a statistical mechanism behind proxy mismatch and reward overoptimization.
\end{itemize}

\section{Related Work}

\paragraph{Reward Modeling and Distribution Shift.}
Reward modeling is a standard component of modern alignment pipelines (RLHF), where a reward signal learned from human feedback is subsequently optimized to produce a policy \citep{christiano2017deep,ziegler2019fine,stiennon2020learning,ouyang2022training}. This downstream optimization changes the learning problem: models are trained on reference data but evaluated on regions heavily emphasized by the optimized policy. For KL-regularized reward maximization, this policy-induced distribution shift is mathematically formalized as an exponential tilting of the reference policy \citep{DBLP:conf/icml/PetersS07, DBLP:conf/icml/GoKKRRD23, DBLP:journals/corr/abs-1910-00177, DBLP:conf/nips/KorbakEKD22}. Consequently, policy-level evaluation translates into an exponentially weighted prediction problem under the reference distribution, drawing natural connections to classical importance weighting under covariate shift \citep{sugiyama2007covariate,gretton2009covariate}. This distribution shift is also connected to reward overoptimization in offline alignment, where optimizing an imperfect reward signal can exploit proxy errors or move beyond data-supported regions \citep{gao2023scaling,huang2024correcting}.

\paragraph{Alignment Theory under Fixed Representations.}
The statistical impact of policy-induced distribution shift has been extensively studied in the context of offline alignment and preference-based RL. Recent theoretical works address this mismatch through coverage or optimal-design assumptions, providing suboptimality, regret, sample-complexity, or query-complexity guarantees in fixed-feature linear reward models, contextual bandits, and offline RLHF settings \citep{zhu2023principled,ji2024reinforcement,scheid2024optimal,DBLP:conf/colt/FosterMR25}. While these frameworks clarify the statistical cost of distribution shift under fixed representations, they circumvent the challenge of feature learning and fitting nonlinear link functions. Our work is complementary to this literature: by moving to a nonlinear feature-learning setting, we explicitly analyze the bottleneck that is inherently absent in fixed linear reward classes.

\paragraph{Single-Index Feature Learning and Nonlinear Transformations.}
A broader line of work studies representation learning beyond fixed kernels or linearization, including one-step feature learning, low-dimensional structure, mean-field training, and spiked-covariance perspectives \citep{DBLP:conf/iclr/BaiL20,ba2022high,DBLP:conf/nips/BaESWW23,DBLP:conf/nips/MahankaliZDG023,DBLP:conf/colt/WangWF24}. Within this landscape, Gaussian single-index models provide a particularly tractable setting for studying nonlinear feature learning \citep{dudeja2018learning}: targets take the form $f^*(x)=\sigma^*(\langle\theta^*,x\rangle)$ with $x\sim\mathcal N(0,I_d)$, and recovery of the hidden direction $\theta^*$ is governed by the Hermite coefficients of the link function. The Information Exponent (IE) captures the first nonzero Hermite signal available to correlational methods \citep{bietti2022learning}. Recent work shows that suitable training dynamics or nonlinear label transformations can reduce the effective exponent from IE toward the Generative Exponent (GE) \citep{lee2024neural,tsiolis2025information,nishikawa2025nonlinear}. Our analysis builds on this IE/GE framework, treating the exponential weighting inherent to reward modeling as a reward-induced nonlinear transformation whose effect is controlled by the policy temperature $\beta$.

\paragraph{Downstream Evaluation vs.\ Prediction Error.}
In two-layer model analyses, the final step often fits the readout layer after fixing the first-layer features, with guarantees stated in terms of prediction error or population risk \citep{rahimi2007random,ba2022high,lee2024neural}. While natural for standard regression, this metric does not directly match the objectives of reward modeling, where the learned reward is exponentiated to deploy a policy. To address this discrepancy, we depart from standard prediction bounds and instead evaluate the fitted network via the true reward obtained under the induced policy. By translating weighted prediction guarantees into tilted-policy regret, we use an evaluation criterion that aligns with the practical endpoints of modern alignment pipelines.

\section{Problem Setup}
\label{sec:problem-setup}

Let $d\in\mathbb N$, let $\mathbb S^{d-1}$ be the unit sphere in
$\mathbb R^d$, and write $\pi_{\mathrm{ref}}:=\mathcal N(0,I_d)$.
Asymptotic notation is with respect to the ambient dimension and sample sizes;
$\lesssim,\gtrsim$ hide constants independent of scaling parameters, and
$\widetilde O(\cdot),\widetilde\Theta(\cdot)$ hide logarithmic factors in the problem parameters.
For a probability measure $\pi$ on $\mathbb R^d$, all $L^p(\pi)$ norms and
inner products are taken with respect to the explicitly displayed distribution;
$ \|f\|_{L^p(\pi)}^p
  :=
  \mathbb E_{x\sim\pi}[f(x)^p].$

We consider a single-index reward model with target direction
$\theta^* \in \mathbb{S}^{d-1}$. Fix a noise radius $\tau>0$. A sample is generated by
\begin{equation}
\label{eq:reward-sample-model}
  y := r^*(x) + \zeta,
  \qquad
  x \sim \pi_{\mathrm{ref}} = \mathcal{N}(0, I_d),
  \qquad
  \zeta \sim \mathrm{Unif}[-\tau, \tau],
\end{equation}
where $r^*(x) = \sigma^*(\langle \theta^*, x \rangle)$ and $\zeta$ is independent of $x$.
Here $\sigma^*$ is a nonconstant upper-bounded
polynomial of degree $q^*\in\mathbb N$, and let
$B_*:=\sup_{u\in\mathbb R}\sigma^*(u)<\infty$.
Here, the upper boundedness of the target link function allows for the analysis of a more realistic setting by excluding situations where extreme outputs yield infinitely high rewards.

For a measurable reward estimate $\widehat r:\mathbb R^d\to\mathbb R$ and a
temperature $\beta>0$, the ideal population target is the prediction error under
the reward-tilted distribution associated with the true reward. For the corresponding tilted distribution $\pi_\beta$ induced by $r^*$, whose
density ratio satisfies
$d\pi_\beta/d\pi_{\mathrm{ref}}\propto e^{r^*(x)/\beta}$, this target is
proportional to the following reference-distribution objective:
\[
  \mathbb E_{x\sim\pi_{\mathrm{ref}},\,\zeta\sim\mathrm{Unif}[-\tau,\tau]}
  \left[
    e^{r^*(x)/\beta}
    \left(y-\widehat r(x)\right)^2
  \right].
\]
Since this ratio depends on the unknown reward $r^*$, we analyze the observable
label-weighted analogue
\[
  \mathcal L_{\mathrm{lbl},\beta}[\widehat r]
  :=
  \mathbb E_{x\sim\pi_{\mathrm{ref}},\,\zeta\sim\mathrm{Unif}[-\tau,\tau]}
  \left[
    e^{y/\beta}
    \left(y-\widehat r(x)\right)^2
  \right].
\]
This loss isolates the exponential weighting effect that later appears in the
first-layer oracle and the label-weighted ridge objective.

As the student model, we consider a two-layer network with $N$ neurons:
\begin{equation}
\label{eq:two-layer-reward-model}
    \textstyle
  r_a(x; W, b)
  :=
  \frac{1}{N}
  \sum_{j=1}^{N}
  a_j \sigma(\langle w_j, x \rangle + b_j).
\end{equation}
where $a = (a_1, \ldots, a_N) \in \mathbb{R}^N$, $W = (w_1, \ldots, w_N) \in (\mathbb{S}^{d-1})^N$, and $b = (b_1, \ldots, b_N) \in \mathbb{R}^N$.
We also assume that $\sigma$ is an
upper-bounded polynomial of degree $q\in\mathbb N$ with $q\ge q^*$, and let
$B:=\sup_{u\in\mathbb R}\sigma(u)<\infty$.
For the feature-recovery phase, we use the standard random initialization
from Gaussian single-index feature-learning analyses \cite{lee2024neural, tsiolis2025information}:
the first-layer directions are initialized independently and uniformly on
the sphere, \(w_j^{(0)}\sim \mathrm{Unif}(\mathbb S^{d-1})\), and the
initial readout coefficients have independent symmetric signs,
\((a_0)_j\sim\mathrm{Unif}\{\pm s_{\rm init}\}\) with
\(s_{\rm init} > 0\).
The first-layer recovery phase does not use random biases. After feature
recovery, we sample the biases used for the second-layer approximation as
\begin{equation}
\label{eq:bias-initialization}
  b_j \sim \mathrm{Unif}[-C_b, C_b],
  \qquad
  C_b = \mathrm{polylog}(d),
  \quad
  j=1,\ldots,N.
\end{equation}

\section{Reward-Weighted Feature Recovery}
\label{sec:reward-weighted-feature-recovery}
The training protocol separates feature recovery from coefficient fitting.
In this section, feature recovery refers to the phase in which the first-layer
directions acquire alignment with the hidden target direction $\theta^*$. Weak
recovery means escaping the random-initialization scale and reaching
$\langle w_j,\theta^*\rangle \ge c_{\mathrm{wk}}$ for a threshold
$c_{\mathrm{wk}}\gtrsim 1/\operatorname{polylog}(d)$; strong recovery refers to
the subsequent regime $\langle w_j,\theta^*\rangle\ge 1-\varepsilon$ for an
accuracy parameter $\varepsilon$.

During feature recovery, we initialize the second-layer coefficients at
$a=a_0\in\mathbb R^N$ and keep them fixed, while the directions $w_j$ are
trained by online spherical SGD \citep{benarous2021online,damian2023smoothing,lee2024neural,DBLP:journals/corr/abs-2405-15459} using the bias-free first-layer features.
Following standard single-index feature-learning analyses, we study the leading
first-layer correlation update induced by the label-weighted objective, which
isolates the teacher-side Hermite signal governing early alignment
\citep{lee2024neural,tsiolis2025information}; Appendix~\ref{app:feature-recovery-details}
derives this oracle from the weighted squared loss in the small-readout regime.
After feature recovery, we sample the biases in \eqref{eq:bias-initialization},
freeze $(w_j,b_j)_{j=1}^N$, and fit only $a\in\mathbb R^N$ with a weighted ridge objective.
We summarize the feature-recovery guarantee here and defer the detailed
recovery statements and proofs to
Appendix~\ref{app:feature-recovery-details}.

\begin{definition}[Information and generative exponents]
\label{def:information-generative-exponents}
Let \(G\sim\mathcal N(0,1)\). For \(f\in L^2(\mathcal N(0,1))\), define
$\mathsf H_i[f]
  :=
  \mathbb E_{G\sim\mathcal N(0,1)}
  \left[f(G)\mathrm{He}_i(G)\right],$ and $i\ge0,$
where \(\mathrm{He}_i\) is the probabilists' Hermite polynomial. The
information exponent and generative exponent are
\[ \textstyle
  \mathrm{IE}(f)
  :=
  \inf\{i\ge1:\mathsf H_i[f]\ne0\},
  \qquad
  \mathrm{GE}(f)
  :=
  \inf_{\Phi\in L^2(P_{f(G)})}\mathrm{IE}(\Phi\circ f),
\]
where \(P_{f(G)}\) is the law of \(f(G)\), with the convention that the
infimum of the empty set is \(\infty\).
\end{definition}

For Gaussian single-index models, feature recovery is governed by the
Hermite expansion of the one-dimensional signal along $\theta^*$
\citep{lee2024neural,tsiolis2025information}. Near random initialization, the
lowest nonzero Hermite degree of the teacher-side signal, together with the
corresponding student-side coefficient, determines the leading drift of
$\langle \theta^*, w_j \rangle$ and hence the weak-recovery sample complexity.
Let $U_i(\beta_1)$ denote the $i$-th Hermite coefficient of the teacher-side
reward-weighted signal; see Appendix~\ref{app:feature-recovery-details} for
its formal definition. To state the teacher-side exponent-reduction bound, let
\[
  p_{\mathrm{gen}}
  :=
  \mathrm{GE}(\sigma^*),
  \qquad
  I_*
  :=
  \min
  \left\{
    I\ge 1:
    \mathrm{IE}\bigl((\sigma^*)^I\bigr)
    =
    p_{\mathrm{gen}}
  \right\}.
\]
The following lemma characterizes how this teacher-side signal scales with the
temperature $\beta_1$.

\begin{lemma}[Dimension-free exponent reduction by exponential weighting]
\label{lem:exponent_reduction}
For every $\beta_1 > 0$, the lower-degree signals vanish:
$U_i(\beta_1) = 0$ for $1 \le i < p_{\mathrm{gen}}$. Moreover, as
$\beta_1\to\infty$,
\[
U_{p_{\mathrm{gen}}}(\beta_1)
=
\frac{
  \mathsf H_{p_{\mathrm{gen}}}[(\sigma^*)^{I_*}]
}{
  (I_*-1)!\,\beta_1^{I_*-1}
}
+
O(\beta_1^{-I_*}).
\]
Consequently, there exist constants $\underline{\beta}_1 > 0$ and
$c_{\mathrm{sig}} > 0$, which are independent of the dimension $d$, such
that for every $\beta_1 \ge \underline{\beta}_1$, the
$p_{\mathrm{gen}}$-degree signal satisfies
\[
|U_{p_\mathrm{gen}}(\beta_1)| \ge c_{\mathrm{sig}} \beta_1^{-(I_*-1)}.
\]
\end{lemma}


The proof is given in Appendix~\ref{app:feature-recovery-details}.
With this signal lower bound established, we can now analyze the online SGD process. 
Lemma \ref{lem:exponent_reduction} provides the magnitude of the teacher-side signal. By applying this signal bound to standard Gaussian SIM feature-learning results for random spherical initialization, a constant fraction of neurons are recovered.

\begin{theorem}[Feature recovery by reward-weighted SGD]
\label{thm:feature-recovery-main}
Assume the student-side nondegeneracy condition
\[
V_{p_{\mathrm{gen}}-1}
:=
\mathbb E_{h\sim\mathcal N(0,1)}
[
\sigma'(h)\mathrm{He}_{p_{\mathrm{gen}}-1}(h)
]
\ne 0 .
\]
Under the random initialization specified in Section~\ref{sec:problem-setup},
for any fixed \(\beta_1\ge \underline{\beta}_1\), the hidden direction $\theta^*$ is recovered for a constant fraction of
neurons with high probability over the initialization and online samples.
The weak-recovery time $T_{\rm wk}$ and the strong-recovery time $T_{\rm str}$ are
\[
T_{\rm wk} = \widetilde O\!\left(
\beta_1^{2(I_*-1)}
d^{(p_{\mathrm{gen}}-1)\vee 1}
\right), \qquad
T_{\rm str} = \widetilde O\!\left(
\beta_1^{2(I_*-1)}d\varepsilon^{-2}
\right).
\]
The precise high-probability statement is given in Appendix~\ref{app:feature-recovery-details}.
\end{theorem}

Note that the threshold $\underline{\beta}_1 = O(1)$ is independent of the dimension $d$, matching the sample complexity of \cite{lee2024neural, tsiolis2025information, nishikawa2025nonlinear} while improving upon the dimension-growing $\mathrm{polylog}(d)$ bounds required by prior clipping-based analyses \cite{nishikawa2025nonlinear}. This is not merely a technical refinement, but addresses a fundamental limitation: an artificially high $\mathrm{polylog}(d)$ temperature would make the exponential weighting nearly uniform, nullifying the benefit of exponent reduction. Lemma \ref{lem:exponent_reduction} makes this explicit. The target learning signal scales as $\beta_1^{-(I_*-1)}$, meaning a $\mathrm{polylog}(d)$ temperature would diminish the signal strength to roughly $1/\mathrm{polylog}(d)$. On the other hand, the existence of an $O(1)$ lower bound itself is practically natural. Before the target feature is reliably represented, setting $\beta_1$ arbitrarily small would cause the update to undesirably over-concentrate on misleading samples.

The role of the lower bound $\beta_1 \ge \underline{\beta}_1$ is to prevent premature temperature decay during the initial feature-learning phase, rather than to restrict the final deployment policy.
In the early stages of training, a premature decay of the temperature would force the model to overfit to a few misleading high-reward samples before the underlying representation is reliably formed. Theorem~\ref{thm:feature-recovery-main} establishes that a dimension-free $O(1)$ temperature is sufficient to prevent this representation bottleneck. Once the hidden direction is successfully recovered, the temperature can be safely lowered to prioritize reward maximization, a regime we analyze in the subsequent section.

\section{Tilted-Policy Value Gap After Feature Recovery}

\subsection{Second-layer weighted ridge setup}
\label{subsec:second-layer-weighted-ridge-setup}
Second-layer ridge fitting starts after feature recovery has produced enough
recovered features. We condition on this recovery event and treat the
first-layer directions as fixed: at least $\tilde\Theta(N)$ of the fixed
directions satisfy $ \left\langle w_j,\theta^* \right\rangle \ge 1-\varepsilon. $

The only remaining optimization variable is the
second-layer coefficient vector $a\in\mathbb R^N$, which we fit using $T_2$
fresh ridge-fitting samples
$\{(x_i,\zeta_i,y_i)\}_{i=1}^{T_2}$ generated independently according to
\eqref{eq:reward-sample-model}.
The goal of this section is to bound the tilted-policy value gap induced by the
fitted reward after feature recovery.

In the reward-modeling pipeline, a learned reward $\widehat r$ is not used
directly; it is exponentiated
and combined with the reference distribution to form a tilted policy. Small
prediction error under $x\sim\pi_{\mathrm{ref}}$ therefore does not, by itself,
control the value of the induced policy, since the exponential weighting can
amplify mild prediction errors in regions of large reward. We therefore
evaluate the fitted reward through tilted-policy value rather than
reference-distribution prediction error \citep{zhu2023principled,gao2023scaling,huang2024correcting}.

In this section, we analyze the downstream policy value using a deployment temperature $\beta_2$. Crucially, $\beta_2$ can be chosen independently of, and typically strictly smaller than, the feature-recovery temperature $\beta_1$. This two-stage setup mirrors practical annealing schedules for KL penalties: a larger $\beta_1$ encourages exploration and stable representation learning, while a smaller $\beta_2$ allows for aggressive reward maximization once the features are frozen.

We compare against a finite-temperature target policy $\pi_{\beta^*}$ rather
than an argmax of $r^*$ because in open-ended generation the desired behavior is
rarely a single reward mode \citep{DBLP:conf/iclr/HoltzmanBDFC20,DBLP:conf/acl/SuzgunMJ23,
DBLP:conf/acl/JinnaiHMZ24}; ${\beta_{2}}$ denotes the
temperature at which the fitted reward is deployed, which need not equal
$\beta^*$. Assuming the policy optimization step is carried out perfectly, we use the exact analytical solution of the KL-regularized reward maximization problem \cite{peters2010relative,DBLP:conf/nips/KorbakEKD22,DBLP:conf/icml/GoKKRRD23} as our deployed policy. For each
${\beta_{2}}>0$, write
\[ \textstyle
  \pi_{\beta_{2}}(dx)
  :=
  \frac{e^{r^*(x)/{\beta_{2}}}}{Z_{\beta_{2}}}\pi_{\mathrm{ref}}(dx),
  \qquad
  Z_{\beta_{2}}
  :=
  \mathbb E_{x\sim\pi_{\mathrm{ref}}}
  \left[
    e^{r^*(x)/{\beta_{2}}}
  \right].
\]

For polynomial reward models, the learned tilted policy may be difficult to
control globally without an additional restriction on where the policy is
evaluated. After feature recovery, the first-layer directions are fixed, and
both \(r^*\) and the fitted reward models considered below depend on \(x\)
only through the directions
$ S:=\operatorname{span}\{\theta^*,w_1,\ldots,w_N\}.$
Let \(P_S\) denote the orthogonal projection onto \(S\).
We therefore truncate
only the coordinates relevant to the reward model. Fix a projected
truncation radius \(R\) with \(R\gtrsim\sqrt{\dim S}\), where the implicit
constant is taken sufficiently large, and write
\[
  B_R:=\{x\in\mathbb R^d:\ \|P_Sx\|_2\le R\}.
\]
This projected truncation keeps the relevant integrability conditions and
normalizers in the tilted-policy value-gap analysis while avoiding a restriction
on the irrelevant \(S^\perp\) directions. For other activation function
classes, this truncation may be unnecessary.
The corresponding truncated target policy is
\[
  \pi_{{\beta_{2}},R}(dx)
  :=
  \frac{\mathbf 1_{B_R}(x)e^{r^*(x)/{\beta_{2}}}}{Z_{{\beta_{2}},R}}
  \pi_{\mathrm{ref}}(dx),
  \qquad
  Z_{{\beta_{2}},R}
  :=
  \mathbb E_{x\sim\pi_{\mathrm{ref}}}
  \left[
    \mathbf 1_{B_R}(x)e^{r^*(x)/{\beta_{2}}}
  \right].
\]
For a learned reward $\widehat r_R$, let
$\widehat\pi_{{\beta_{2}},R}$ and $\widehat Z_{{\beta_{2}},R}$ denote the
policy and normalizing constant obtained from the same construction after
replacing $r^*$ by $\widehat r_R$.
The natural performance measure for the fitted reward is the value gap
between the target policy at temperature $\beta^*$ and this learned truncated tilted
policy at temperature ${\beta_{2}}$,
\[
  \mathcal R_R
  :=
  \mathbb E_{x\sim\pi_{\beta^*}}[r^*(x)]
  -
  \mathbb E_{x\sim\widehat\pi_{{\beta_{2}},R}}[r^*(x)].
\]

Adding and subtracting $\mathbb E_{x\sim\pi_{\beta_{2}}}[r^*(x)]$ and
$\mathbb E_{x\sim\pi_{{\beta_{2}},R}}[r^*(x)]$ decomposes $\mathcal R_R$ into three
conceptually distinct contributions:
\[
  \mathcal R_R
  =
  T_{\mathrm{temp}}
  +
  T_{\mathrm{cut}}(R)
  +
  T_{\mathrm{learn}}(R),
\]
where
\[
  T_{\mathrm{temp}}
  :=
  \mathbb E_{x\sim\pi_{\beta^*}}[r^*(x)]
  -
  \mathbb E_{x\sim\pi_{\beta_{2}}}[r^*(x)],
  \quad
  T_{\mathrm{cut}}(R)
  :=
  \mathbb E_{x\sim\pi_{\beta_{2}}}[r^*(x)]
  -
  \mathbb E_{x\sim\pi_{{\beta_{2}},R}}[r^*(x)],
\]
and
\[
  T_{\mathrm{learn}}(R)
  :=
  \mathbb E_{x\sim\pi_{{\beta_{2}},R}}[r^*(x)]
  -
  \mathbb E_{x\sim\widehat\pi_{{\beta_{2}},R}}[r^*(x)].
\]
The first two terms depend only on $r^*$, $\beta^*$, ${\beta_{2}}$, and $R$, and
describe properties of the target tilted policy that are independent of the
learning procedure; we control them in Section~\ref{subsec:temperature_mismatch_and_truncation}. Only
$T_{\mathrm{learn}}(R)$ depends on the second-layer fitting rule, and bounding
it through a weighted reward prediction error is the subject of Section~\ref{subsec:learned_term_bound}. 

\subsection{Temperature mismatch and truncation} \label{subsec:temperature_mismatch_and_truncation}
Define $ m_{\beta_{2}}
  :=
  \mathbb E_{x\sim\pi_{\beta_{2}}}[r^*(x)]$ and 
  $ p_{{\beta_{2}},R}
  :=
  \pi_{\beta_{2}}(B_R).$
The identities used in the next two propositions are
\[
  T_{\mathrm{temp}}
  =
  \int_{1/{\beta_{2}}}^{1/\beta^*}
  \operatorname{Var}_{x\sim\pi_{\beta'}}[r^*(x)]
  \,d\left(\frac{1}{\beta'}\right),
  \quad
  T_{\mathrm{cut}}(R)
  =
  \frac{1}{p_{{\beta_{2}},R}}
  \mathbb E_{x\sim\pi_{\beta_{2}}}
  \left[
    \left(r^*(x)-m_{\beta_{2}}\right)\mathbf 1_{B_R^c}(x)
  \right].
\]

The first two terms in the value-gap decomposition describe the target tilted
policy alone, and both can be controlled at low temperature once the local
behavior of $\sigma^*$ near its maximizers is understood. 
Let $u_1,\ldots,u_m$ be the global maximizers of $\sigma^*$, and let $p_i$ be
the local order at $u_i$, meaning
$\sigma^*(u_i+t)=B_*-c_i t^{2p_i}+o(t^{2p_i})$ for some $c_i>0$. Define
$ p_{\max}:=\max_{1\le i\le m}p_i$ and $I_{\max}:=\{i:\ p_i=p_{\max}\}.$
For a constant $\kappa>0$ determined by these local orders and defined in
Appendix~\ref{app:one-dimensional-tilted-limit}, the low-temperature variance
expansion is
\[
  \operatorname{Var}_{x\sim\pi_{\beta_{2}}}[r^*(x)]
  =
  \frac{{\beta_{2}}^2}{2p_{\max}}
  +
  O({\beta_{2}}^{2+\kappa}),
\]
from which the following common bound follows by standard low-temperature and
truncation estimates; the proof is given in
Appendix~\ref{app:proof-temperature-mismatch}.


\begin{proposition}[Common low-temperature non-learning terms]
\label{prop:low-temp-temperature-mismatch}
\label{prop:low-temp-truncation}
For every sufficiently small
\({\overline{\beta}_{2}}>0\), there exist constants
\(C_{\mathrm{temp}}>0\) and \(C_{R,{\overline{\beta}_{2}}}<\infty\) such that
\[
  |T_{\mathrm{temp}}|
  \le
  \left(
    \frac{1}{2p_{\max}}
    +
    C_{\mathrm{temp}}{\overline{\beta}_{2}}^\kappa
  \right)
  |\beta^*-{\beta_{2}}|,
  \qquad
  |T_{\mathrm{cut}}(R)|
  \le
  C_{R,{\overline{\beta}_{2}}}{\beta_{2}},
\]
for all \(0<{\beta_{2}},\beta^*\le{\overline{\beta}_{2}}\).
\end{proposition}

\subsection{Label/surrogate-weighted ridge fitting}
\label{subsec:learned_term_bound}

For \(a\in\mathbb R^N\), write
$ 
  r_a(x)
  :=
  \frac{1}{N}
  \sum_{j=1}^N a_j\sigma(\langle w_j,x\rangle+b_j).
$
Given a positive weighting rule \(W(x_i,y_i)\), define the weighted ridge
estimator
\[
  \hat a_R^W
  \in
  \arg\min_{a\in\mathbb R^N}
  \left\{
    \frac{1}{T_2}
    \sum_{i=1}^{T_2}
    \mathbf 1_{B_R}(x_i)W(x_i,y_i)
    \left(y_i-r_a(x_i)\right)^2
    +
    \lambda\|a\|_2^2
  \right\},
  \qquad
  \widehat r_R^W:=r_{\hat a_R^W}.
\]

We study two choices of \(W\), denoted by
\(\mathrm w\in\{\mathrm{lbl},\mathrm{surr}\}\).

\begin{itemize}[leftmargin=*]
  \item \textbf{Label-weighted fitting.}
  The label-weighted scheme uses the observed reward label in the weight,
  $ W_{\mathrm{lbl}}(x_i,y_i):=e^{y_i/{\beta_{2}}}.$
  This is the cleanest theoretical benchmark for reward-induced exponential
  weighting. We write
  $\widehat r_R^{\mathrm{lbl}}:=\widehat r_R^{W_{\mathrm{lbl}}}$
  and
  $\nu_R^{\mathrm{lbl}}:=\pi_{{\beta_{2}},R}.$
  The learned reward induces
  \[
    \widehat\pi_{{\beta_{2}},R}^{\mathrm{lbl}}(dx)
    :=
    \left(\widehat Z_{{\beta_{2}},R}^{\mathrm{lbl}}\right)^{-1}
    \mathbf 1_{B_R}(x)e^{\widehat r_R^{\mathrm{lbl}}(x)/{\beta_{2}}}
    \pi_{\mathrm{ref}}(dx),
  \]
  where \(\widehat Z_{{\beta_{2}},R}^{\mathrm{lbl}}\) is the normalizing
  constant. This is the learned policy used in \(\mathcal R_R\) and
  \(T_{\mathrm{learn}}(R)\).

  \item \textbf{Surrogate-weighted fitting.}
The surrogate-weighted scheme is closer to plug-in reward-modeling
pipelines, where sample weights are computed from a learned surrogate
reward rather than from the observed reward labels. Treating the surrogate
and the fitted weights as jointly changing makes the objective highly
nonconvex and theoretically intractable. A natural way to avoid this is to
update the model in blocks: fix the surrogate reward for \(T_2/K\) samples,
update the target weights, and iterate this process \(K\) times.

For analytical clarity, we state the single-block version, corresponding
to \(K=1\): the initial readout vector \(a_0\) from the feature-recovery
phase is frozen and treated as fixed throughout the \(T_2\)
ridge-fitting samples. This choice keeps the notation and proof structure
focused on one weighted ridge problem. The same argument can be applied
block by block in the \(K\)-block sequential setting. With
  $ 
    r_{a_0}(x)
    :=
    \frac{1}{N}\sum_{j=1}^N
    (a_0)_j\sigma(\langle w_j,x\rangle+b_j),
  $
  the surrogate-weighted rule is $W_{\mathrm{surr}}(x_i,y_i):=e^{r_{a_0}(x_i)/{\beta_{2}}}.$
    We write $\widehat r_R^{\mathrm{surr}}:=\widehat r_R^{W_{\mathrm{surr}}}$ and define
  \[
    \nu_R^{\mathrm{surr}}(dx)
    :=
    \left(Z_{{\beta_{2}},R}^{a_0}\right)^{-1}
    \mathbf 1_{B_R}(x)e^{r_{a_0}(x)/{\beta_{2}}}
    \pi_{\mathrm{ref}}(dx),
  \]
  where \(Z_{{\beta_{2}},R}^{a_0}\) is the corresponding normalizing constant.
    Let $\widehat\pi_{{\beta_{2}},R}^{\mathrm{surr}}$ be the policy obtained from
  the same truncated-tilt construction with $r_{a_0}$ replaced by
  $\widehat r_R^{\mathrm{surr}}$, and let
  $\widehat Z_{{\beta_{2}},R}^{\mathrm{surr}}$ denote its normalizing constant.
  The corresponding surrogate-weighted value gap is
  \[
    \mathcal R_R^{\mathrm{surr}}
    :=
    \mathbb E_{x\sim\pi_{\beta^*}}[r^*(x)]
    -
    \mathbb E_{x\sim\widehat\pi_{{\beta_{2}},R}^{\mathrm{surr}}}[r^*(x)].
  \]
\end{itemize}

We next relate the weighted prediction error of each fitted reward to the
corresponding learning term in the policy value decomposition. For each
\(\mathrm w\in\{\mathrm{lbl},\mathrm{surr}\}\), define the interpolation path
\[
  \pi_{{\beta_{2}},t,R}^{\mathrm w}(dx)
  :=
    \left(Z_{{\beta_{2}},t,R}^{\mathrm w}\right)^{-1}
    \mathbf 1_{B_R}(x)
    e^{(r^*(x)+t(\widehat r_R^{\mathrm w}(x)-r^*(x)))/{\beta_{2}}}
  \pi_{\mathrm{ref}}(dx),
  \qquad t\in[0,1],
\]
where \(Z_{{\beta_{2}},t,R}^{\mathrm w}\) is the normalizing constant, and set
\[
  \mathcal D_{\mathrm w,R}
  :=
  \int_0^1
  \left\|
    \frac{d\pi_{{\beta_{2}},t,R}^{\mathrm w}}{d\nu_R^{\mathrm w}}
  \right\|_{L^2(\nu_R^{\mathrm w})}
  \,dt.
\]

\begin{lemma}[Density-ratio bridge for the learning term]
\label{lem:learning-bridge}
Let \(M_R:=\sup_{x\in B_R}|r^*(x)|\). For each
\(\mathrm w\in\{\mathrm{lbl},\mathrm{surr}\}\), define
$
  T_{\mathrm{learn}}^{\mathrm w}(R)
  :=
  \mathbb E_{x\sim\pi_{{\beta_{2}},R}}[r^*(x)]
  -
  \mathbb E_{x\sim\widehat\pi_{{\beta_{2}},R}^{\mathrm w}}[r^*(x)].
$
Then
\[ \textstyle
  |T_{\mathrm{learn}}^{\mathrm w}(R)|
  \le
  2{\beta_{2}}^{-1}M_R\mathcal D_{\mathrm w,R}
  \inf_{c\in\mathbb R}
  \|\widehat r_R^{\mathrm w}-r^*-c\|_{L^2(\nu_R^{\mathrm w})}.
\]
\end{lemma}

The quantity \(\mathcal D_{\mathrm w,R}\) measures how well the weighted
regression distribution \(\nu_R^{\mathrm w}\) covers the interpolation path
from the target tilted policy to the learned policy. It plays a similar
role to coverage/concentrability quantities in offline alignment
analyses \citep{zhu2023principled,huang2024correcting}. 

    The infimum over \(c\) reflects the fact that tilted policies are invariant to
global reward shifts. This shift-invariant form is useful for label weighting:
at the population level, the label-weighted squared loss targets
\(r^*+m_{\zeta,\beta_2}\) rather than \(r^*\) (see Appendix~\ref{app:label-shift-calculation}), where
\[ \textstyle
  m_{\zeta,\beta_2}
  :=
  \mathbb E_{\zeta\sim\nu_{\zeta,\beta_2}}[\zeta],
  \qquad
  \nu_{\zeta,\beta_2}(d\zeta)
  :=
  \left(
    \mathbb E_{\zeta'\sim\mathrm{Unif}[-\tau,\tau]}
    [e^{\zeta'/\beta_2}]
    \right)^{-1}
  e^{\zeta/\beta_2} 
  \mathrm{Unif}[-\tau,\tau](d\zeta).
\]
This additive shift is absorbed by the infimum over \(c\) in
Lemma~\ref{lem:learning-bridge}.

\subsection{Value-gap bounds under projected truncation}
\label{subsec:regret_bounds_projected_truncation}

In this subsection, assume the first-layer recovery condition from the
second-layer weighted ridge setup and fix 
$  0<{\beta_{2}},\beta^*\le {\overline{\beta}_{2}},$
where ${\overline{\beta}_{2}}>0$ is sufficiently small. Define the temperature and
truncation contribution
\[
  \Gamma_{R,{\overline{\beta}_{2}}}({\beta_{2}},\beta^*)
  :=
  \left(
    \frac{1}{2p_{\max}}
    +
    C_{\mathrm{temp}}{\overline{\beta}_{2}}^\kappa
  \right)
  |\beta^*-{\beta_{2}}|
  +
  C_{R,{\overline{\beta}_{2}}}{\beta_{2}}.
\]
Throughout this subsection, $\lesssim$ hides constants depending only on the
fixed problem parameters, the truncation radius $R$, the frozen first-layer
directions, and the upper temperature cutoff ${\overline{\beta}_{2}}$ for simplicity.

\begin{theorem}[Label-weighted projected-truncation value gap]
\label{thm:final-label-weighted-regret}
Define \(\alpha:=1/(2p_{\max})\). Choose the regularization parameter
\(\lambda\) at the label-weighted ridge scale specified in
\eqref{eq:label-weighted-lambda-choice}. Then, with probability at
least \(1-4\delta_0\),
\[
  |\mathcal R_R|
  \lesssim\;
  \Gamma_{R,{\overline{\beta}_{2}}}({\beta_{2}},\beta^*)
  +
  M_R\mathcal D_{\mathrm{lbl},R}
  \left[
    {\beta_{2}}^{-(\alpha+3)/2}(N^{-1}+\varepsilon)
    +
      {\beta_{2}}^{-1-(\alpha+1)/4}
      (T_2\delta_0)^{-1/4}
  \right].
\]
\end{theorem}

\begin{theorem}[Surrogate-weighted projected-truncation value gap]
\label{thm:final-surrogate-weighted-regret}
Define
\[ \textstyle
  M_{0,R}:=\sup_{x\in B_R}r_{a_0}(x),
  \qquad
  \mathcal C_{0,R}({\beta_{2}})
  :=
  \left(Z_{{\beta_{2}},R}^{a_0}\right)^{-1} e^{M_{0,R}/{\beta_{2}}}.
\]
Choose the regularization parameter \(\lambda\) at the surrogate-weighted
ridge scale specified in \eqref{eq:surrogate-weighted-lambda-choice}.
Then, with probability at least \(1-4\delta_0\),
\[
  |\mathcal R_R^{\mathrm{surr}}|
  \lesssim\;
  \Gamma_{R,{\overline{\beta}_{2}}}({\beta_{2}},\beta^*)
  +
  M_R\mathcal D_{\mathrm{surr},R}
  \left[
    \mathcal C_{0,R}({\beta_{2}})^{1/2}
    {\beta_{2}}^{-1}
    (N^{-1}+\varepsilon)
    +
      \mathcal C_{0,R}({\beta_{2}})^{1/4}
      {\beta_{2}}^{-1}(T_2\delta_0)^{-1/4}
  \right].
\]
\end{theorem}

Theorem~\ref{thm:final-label-weighted-regret} and
Theorem~\ref{thm:final-surrogate-weighted-regret} follow by combining
Lemma~\ref{lem:learning-bridge} with the common non-learning bound in
Proposition~\ref{prop:low-temp-temperature-mismatch}, and then applying the
corresponding weighted ridge learning bound.

Combining Theorem~\ref{thm:final-surrogate-weighted-regret} with
Lemma~\ref{lem:surrogate-near-maximizer-mass} in
Appendix~\ref{app:surrogate-partition} gives the explicit polynomial
upper bound \(\mathcal C_{0,R}(\beta_2)\lesssim\beta_2^{-d_S}\), where
\(d_S:=\dim S\). More refined local information about the maxima of
\(r_{a_0}\) can yield a smaller exponent
\(\alpha_0\le d_S\); see Remark~\ref{rem:sharper_exponents}.

\begin{corollary}[Polynomial surrogate-weighted value-gap bound]
\label{cor:surrogate-weighted-projected-truncation}
In the setting of Theorem~\ref{thm:final-surrogate-weighted-regret}, let
\(d_S:=\dim S\). Then, for all sufficiently small \({\beta_{2}}>0\), with
probability at least \(1-4\delta_0\),
\[
\begin{aligned}
  |\mathcal R_R^{\mathrm{surr}}|
  \lesssim\;&
  \Gamma_{R,{\overline{\beta}_{2}}}({\beta_{2}},\beta^*)
  +
  M_R\mathcal D_{\mathrm{surr},R}
  \left[
    {\beta_{2}}^{-1-d_S/2}(N^{-1}+\varepsilon)
    +
      {\beta_{2}}^{-1-d_S/4}
      (T_2\delta_0)^{-1/4}
  \right].
\end{aligned}
\]
\end{corollary}

The three bounds decompose the policy value gap into
interpretable components. The common term
\(\Gamma_{R,\overline{\beta}_{2}}(\beta_{2},\beta^*)\) captures the
temperature mismatch and projected-truncation costs, which are independent
of the learning rule. The remaining terms are learning costs, scaled by
\(\mathcal D_{\mathrm{lbl},R}\) or \(\mathcal D_{\mathrm{surr},R}\). 
Inside the brackets, the approximation error is proportional to
\(N^{-1}+\varepsilon\), reflecting how well \(r^*\) can be approximated
using the recovered features, while the finite-sample error decays as
\((T_2\delta_0)^{-1/4}\).

The difference between label and surrogate weighting is the extra factor
\(\mathcal C_{0,R}(\beta_{2})\) in the surrogate learning cost. As defined
in Theorem~\ref{thm:final-surrogate-weighted-regret},
\(\mathcal C_{0,R}(\beta_{2})\) upper bounds the density of the
surrogate-induced measure \(\nu_R^{\mathrm{surr}}\) relative to
\(\pi_{\mathrm{ref}}\) on \(B_R\). When this factor is large, the surrogate
weights may concentrate mass on a narrow high-surrogate-reward region. This
makes the ridge objective emphasize errors on that region and can also reduce
the effective sample efficiency. Consequently, both the approximation error
and the finite-sample error are amplified in the surrogate-weighted bound.

\paragraph{Choosing the deployment temperature.}
The value-gap bounds suggest a resource-aware choice of the deployment
temperature. Smaller \(\beta_2\) better matches a low-temperature target
policy, but also amplifies the learning terms through exponential
weighting and the coverage quantity \(\mathcal D_{\mathrm w,R}\). Since
\(\mathcal D_{\mathrm w,R}\) depends on the fitted reward, we choose
\(\beta_2\) using a deterministic coverage envelope
\(\overline{\mathcal D}_{\mathrm w,R}\). For a tolerance \(\eta>0\), let
\(\mathcal S_w^{\mathrm{adm}}(\eta)\subset(0,\overline\beta_2]\) be the
set of temperatures for which the learning-term upper bound is at most
\(\eta\). Choosing
$
  \beta_{2,\eta}^{\mathrm w}
  \in
  \arg\min_{\beta\in\mathcal S_w^{\mathrm{adm}}(\eta)}
  \Gamma_{R,\overline\beta_2}(\beta,\beta^\ast)
$
gives
\[ \textstyle
  |\mathcal R_R^{\mathrm w}|
  \lesssim
  \inf_{\beta\in\mathcal S_w^{\mathrm{adm}}(\eta)}
  \Gamma_{R,\overline\beta_2}(\beta,\beta^\ast)
  +
  \eta.
\]
Thus the target temperature is used when it is admissible under the
coverage envelope; otherwise one chooses the admissible temperature with
the best temperature-mismatch and truncation tradeoff.
Appendix~\ref{app:temperature-selection} gives explicit forms of
\(\mathcal S_w^{\mathrm{adm}}(\eta)\) and its resource dependence.

\section{Conclusion and Discussion}

This paper studied reward modeling in a controlled Gaussian single-index setting, with emphasis on how reward prediction interacts with downstream tilted-policy value. The main point is that the relevant statistical target is not only prediction accuracy under $x\sim\pi_{\mathrm{ref}}$, but the tilted-policy value of the learned truncated tilted policy. We analyzed this interaction through first-layer feature recovery under an exponentially weighted reward oracle and second-layer weighted ridge fitting.

For feature recovery, we examined the implications of the IE/GE theory for the reward-induced exponential oracle.
The exponential weighting changes the Hermite signal seen by the first layer, and the analysis makes explicit how this signal depends on the policy temperature ${\beta_{1}}$. 
In our setting, the lower bound on ${\beta_{1}}$ is independent of the dimension $d$. Crucially, while matching the sample complexity of \citep{lee2024neural,tsiolis2025information,nishikawa2025nonlinear}, this improves the polylogarithmic-in-$d$ scale required by clipping-based analyses \citep{nishikawa2025nonlinear} to a fixed $O(1)$ bound.
This bound is consistent with the interpretation that, before first-layer feature recovery, too small a temperature may over-emphasize high-reward samples before the relevant feature is reliably represented.

For second-layer ridge fitting, we translated weighted reward prediction guarantees into tilted-policy value bounds. After first-layer feature recovery, the frozen network reduces the remaining problem to a weighted ridge fit over the second-layer coefficients. We analyzed both a label-weighted objective, which serves as a clean theoretical benchmark for reward-induced exponential weighting, and a surrogate-weighted objective, which is closer to the plug-in reward-modeling pipeline. In both cases, the final bounds decompose the tilted-policy value gap into a temperature mismatch term, a projected-truncation term, and a learning term controlled by weighted reward prediction error. The dependence on ${\beta_{2}}$ is kept explicit throughout the ridge-fitting analysis.

While the Gaussian single-index model is a theoretical abstraction, it provides a foundational characterization of the interaction between reward modeling and policy value. Future work could build upon this framework by relaxing several technical assumptions, such as extending the analysis to richer input distributions or multi-index reward structures beyond the current Gaussian setting. Another promising direction is to remove the projected truncation to further generalize the theoretical results.

\begin{ack}

RH and RK were partially supported by JSPS KAKENHI (25H01107) and JST BOOST (JPMJBS2418).
TS was partially supported by JSPS KAKENHI (24K02905) and JST CREST (JPMJCR2015).
This research is supported by the National Research Foundation, Singapore and the Ministry of Digital Development and Information under the AI Visiting Professorship Programme (award number AIVP-2024-004). Any opinions, findings and conclusions or recommendations expressed in this material are those of the author(s) and do not reflect the views of National Research Foundation, Singapore and the Ministry of Digital Development and Information.
\end{ack}

\medskip

\bibliographystyle{plainnat}
\bibliography{ref}


\appendix

\section{Reward-weighted feature recovery}
\label{app:feature-recovery-details}

This appendix gives the detailed feature-recovery statements summarized by
Theorem~\ref{thm:feature-recovery-main}. The analysis leverages the
online spherical SGD framework of \citet{tsiolis2025information} to treat the
exponentially weighted reward oracle induced by our objective.

We use the information exponent, generative exponent, and Hermite-coefficient
notation from Definition~\ref{def:information-generative-exponents}. For the
target activation, write
\[
  p_{\mathrm{gen}}
  :=
  \mathrm{GE}(\sigma^*).
\]
Since $\sigma^*$ is a nonconstant polynomial, the polynomial-link
characterization of \citet[Proposition 6 and Lemma 8]{lee2024neural}
implies that this level is attained by a power transformation; see
also \citet[Lemma 2.3]{tsiolis2025information}. We therefore write
\[
  I_*
  :=
  \min
  \left\{
    I\ge 1:
    \mathrm{IE}\bigl((\sigma^*)^I\bigr)
    =
    p_{\mathrm{gen}}
  \right\}.
\]

During feature recovery, the second-layer coefficients are kept fixed at $a_0$, and each direction is updated by online SGD with spherical projection.
The update below is the leading correlation part of the label-weighted squared
loss gradient near initialization \cite{DBLP:conf/colt/AbbeAM23,lee2024neural,tsiolis2025information}. Indeed, for
\[
  \mathcal L_{\beta_1}[W,a]
  :=
  \mathbb E_{x\sim\pi_{\mathrm{ref}},\,\zeta\sim\mathrm{Unif}[-\tau,\tau]}
  \left[
    e^{y/{\beta_1}}(y-r_a(x;W))^2
  \right],
\]
the spherical first-layer gradient-descent direction for neuron $j$, up to the
common positive scale $2/N$, is
\[
  a_j
  \mathbb E_{x,\zeta}
  \left[
    e^{y/{\beta_1}}(y-r_a(x;W))
    \sigma'(\langle w_j,x\rangle)P_{w_j}^{\perp}x
  \right].
\]
Decomposing
$e^{y/{\beta_1}}(y-r_a(x;W))=y e^{y/{\beta_1}}-r_a(x;W)e^{y/{\beta_1}}$
separates the leading term
\[
  a_j
  \mathbb E_{x,\zeta}
  \left[
    y e^{y/{\beta_1}}
    \sigma'(\langle w_j,x\rangle)P_{w_j}^{\perp}x
  \right]
\]
from the correction involving the current network output. In the small-readout
regime used for initial feature learning, this leading correlation term is the
teacher-side signal analyzed by the single-index oracle.
For the coefficient-level statement below, the relevant population oracle is the exponentially weighted label transform
\[
  G_{\beta_1}(u)
  :=
  \mathbb E_{\zeta\sim\mathrm{Unif}[-\tau,\tau]}
  \left[
    (u+\zeta)e^{(u+\zeta)/{\beta_1}}
  \right].
\]
For each $i\ge 1$, define the teacher-side signal coefficient
\[
  U_i({\beta_1})
  :=
  \mathbb E_{g\sim\mathcal N(0,1)}
  \left[
    G_{\beta_1}(\sigma^*(g))\mathrm{He}_i(g)
  \right].
\]
For a fixed second-layer coefficient $a$, define the student-side coefficient
\[
V_{i-1}
:=
\mathbb E_{h\sim\mathcal N(0,1)}
\left[
  \sigma'(h)\mathrm{He}_{i-1}(h)
\right],
\]
and set
\[
\mu_i({\beta_1}, a) := i \, a \, U_i({\beta_1}) V_{i-1}.
\]

\begin{proof}[Proof of Lemma~\ref{lem:exponent_reduction}]
The label transform $G_{\beta_1}$ can be written as
\[
  G_{\beta_1}(u)
  =
  e^{u/\beta_1}
  \left(
    u\,\mathbb E_{\zeta\sim\mathrm{Unif}[-\tau,\tau]}
    [e^{\zeta/\beta_1}]
    +
    \mathbb E_{\zeta\sim\mathrm{Unif}[-\tau,\tau]}
    [\zeta e^{\zeta/\beta_1}]
  \right).
\]
Since $\sigma^*$ is an upper-bounded polynomial and $\zeta$ is bounded,
$G_{\beta_1}\in L^2(P_{\sigma^*(g)})$. Therefore
$G_{\beta_1}\circ\sigma^*$ is an admissible label transform in the
definition of $\mathrm{GE}(\sigma^*)$, and the definition gives
$U_i(\beta_1)=0$ for $1\le i<p_{\mathrm{gen}}$.

It remains to identify the first nonzero coefficient. Let
$p=p_{\mathrm{gen}}$, $t=\beta_1^{-1}$, and
$S=\sigma^*(g)+\zeta$. By Taylor's formula with remainder,
\[
  S e^{tS}
  =
  \sum_{m=0}^{I_*-1}
  \frac{t^mS^{m+1}}{m!}
  +
  R_{I_*}(t,S),
\]
where, if $\sigma^*(g)\le B_*$,
\[
  |R_{I_*}(t,S)|
  \le
  \frac{t^{I_*}}{I_*!}
  |S|^{I_*+1}e^{t(B_*+\tau)_{+}}.
\]
The right-hand side is integrable after multiplication by
$|\mathrm{He}_p(g)|$, because $S$ is a polynomial of the Gaussian
variable plus bounded noise. Hence the remainder contributes
$O(\beta_1^{-I_*})$ to $U_p(\beta_1)$.

For $m<I_*-1$, the term
$\mathbb E_{\zeta\sim\mathrm{Unif}[-\tau,\tau]}[S^{m+1}]$ is a
polynomial in $\sigma^*(g)$ involving only powers
$(\sigma^*)^k$ with $k<I_*$. By the definition of $I_*$ and the fact that
$p$ is the generative exponent, each such power has zero $p$-th Hermite
coefficient. For $m=I_*-1$,
\[
  \mathbb E_{\zeta\sim\mathrm{Unif}[-\tau,\tau]}[S^{I_*}]
  =
  (\sigma^*)^{I_*}
  +
  \text{lower powers of }\sigma^*,
\]
and the lower powers again have zero $p$-th Hermite coefficient. Thus
\[
  U_p(\beta_1)
  =
  \frac{
    \mathsf H_p[(\sigma^*)^{I_*}]
  }{
    (I_*-1)!\,\beta_1^{I_*-1}
  }
  +
  O(\beta_1^{-I_*}).
\]
The leading coefficient is nonzero by the definition of $I_*$. The
displayed lower bound follows for all sufficiently large $\beta_1$.
\end{proof}

\begin{theorem}[Weak recovery for the exponentially weighted first-layer oracle]
\label{thm:first-layer-weak-recovery}
For any failure parameter \(\delta\in(0,1)\), there exists a constant
\(C_{\mathrm{wk}}>0\) such that the following holds.
Consider the feature-recovery update with step size \(\eta_1>0\):
\[
  w_j^{(t+1)}
  =
  \frac{
    w_j^{(t)}
    +
    \eta_1
    (a_0)_j
    y_t e^{y_t/{\beta_1}}
    \sigma'(\langle w_j^{(t)},x_t\rangle)
    P_{w_j^{(t)}}^\perp x_t
  }{
    \left\|
    w_j^{(t)}
    +
    \eta_1
    (a_0)_j
    y_t e^{y_t/{\beta_1}}
    \sigma'(\langle w_j^{(t)},x_t\rangle)
    P_{w_j^{(t)}}^\perp x_t
    \right\|_2
  },
\]
where $(x_t,\zeta_t)$ are i.i.d. samples from
$\pi_{\mathrm{ref}}\otimes\mathrm{Unif}[-\tau,\tau]$,
$y_t=r^*(x_t)+\zeta_t$, and
$P_w^\perp:=I_d-ww^\top$.
Fix a neuron $j$ and suppose that
\[
  \mu_{p_{\mathrm{gen}}}({\beta_1},(a_0)_j)>0.
\]
If
\[
  \eta_1
  \le
  C_{\mathrm{wk}}
  \delta\,
  \mu_{p_{\mathrm{gen}}}({\beta_1},(a_0)_j)
  d^{-(p_{\mathrm{gen}}/2\vee 1)},
\]
then, conditional on an initialization satisfying
$\langle w_j^{(0)},\theta^*\rangle\asymp d^{-1/2}$, the iterate reaches weak recovery,
\[
  \langle w_j^{(T_{\mathrm{wk}})},\theta^*\rangle
  \ge
  c_{\mathrm{wk}},
  \qquad
  c_{\mathrm{wk}}\gtrsim 1/\mathrm{polylog}\,d,
\]
with probability at least $1-\delta$ after
\[
  T_{\mathrm{wk}}
  =
  \widetilde O
  \left(
    \eta_1^{-1}
    \mu_{p_{\mathrm{gen}}}({\beta_1},(a_0)_j)^{-1}
    d^{((p_{\mathrm{gen}}-2)/2)\vee 0}
  \right)
\]
iterations.
With the largest stable step size, this becomes
\[
  T_{\mathrm{wk}}
  =
  \widetilde O
  \left(
    \mu_{p_{\mathrm{gen}}}({\beta_1},(a_0)_j)^{-2}
    d^{(p_{\mathrm{gen}}-1)\vee 1}
  \right).
\]
\end{theorem}

\begin{corollary}[Large-fixed-temperature weak recovery dependence on ${\beta_1}$]
\label{cor:first-layer-large-fixed-temperature-weak-recovery}
Under the assumptions of Theorem~\ref{thm:first-layer-weak-recovery}, suppose additionally that $|(a_0)_jV_{p_{\mathrm{gen}}-1}|\ge c_{\mathrm{stu}}>0$ and the sign of $(a_0)_jV_{p_{\mathrm{gen}}-1}$ is compatible with $U_{p_{\mathrm{gen}}}({\beta_1})$.
Then for every ${\beta_1} \ge \underline{\beta}_1$, the largest-stable-step weak recovery time satisfies
\[
  T_{\mathrm{wk}}
  =
  \widetilde O
  \left(
    {\beta_1}^{2(I_*-1)}
    d^{(p_{\mathrm{gen}}-1)\vee 1}
  \right).
\]
\end{corollary}

\begin{theorem}[Strong recovery after weak recovery]
\label{thm:first-layer-strong-recovery}
For any failure parameter \(\delta\in(0,1)\), there exists a constant
\(C_{\mathrm{str}}>0\) such that the following holds. Work under the same
feature-recovery update and coefficient condition as in
Theorem~\ref{thm:first-layer-weak-recovery}. Suppose that a neuron has already
reached weak recovery,
\[
  \langle w_j^{(0)},\theta^*\rangle
  \ge
  c_{\mathrm{wk}},
  \qquad
  c_{\mathrm{wk}}\gtrsim 1/\mathrm{polylog}\,d.
\]
For any $\varepsilon\in(0,1)$, if
\[
  \eta_1
  \le
  C_{\mathrm{str}}
  \delta\,
  d^{-1}
  \varepsilon
  \mu_{p_{\mathrm{gen}}}({\beta_1},(a_0)_j)
  c_{\mathrm{wk}}^{p_{\mathrm{gen}}-1},
\]
then the iterate reaches strong recovery,
\[
  \langle w_j^{(T_{\mathrm{str}})},\theta^*\rangle
  \ge
  1-\varepsilon,
\]
with probability at least $1-\delta$ after
\[
  T_{\mathrm{str}}
  =
  \widetilde O
  \left(
    \eta_1^{-1}
    \varepsilon^{-1}
    \mu_{p_{\mathrm{gen}}}({\beta_1},(a_0)_j)^{-1}
  \right)
\]
additional iterations.
With the largest stable step size, this gives
\[
  T_{\mathrm{str}}
  =
  \widetilde O
  \left(
    d\,
    \varepsilon^{-2}
    c_{\mathrm{wk}}^{-(p_{\mathrm{gen}}-1)}
    \mu_{p_{\mathrm{gen}}}({\beta_1},(a_0)_j)^{-2}
  \right).
\]
\end{theorem}

\begin{corollary}[Large-fixed-temperature strong recovery dependence on ${\beta_1}$]
\label{cor:first-layer-large-fixed-temperature-strong-recovery}
Under the assumptions of Theorem~\ref{thm:first-layer-strong-recovery}, suppose additionally that $|(a_0)_jV_{p_{\mathrm{gen}}-1}|\ge c_{\mathrm{stu}}>0$ and the sign of $(a_0)_jV_{p_{\mathrm{gen}}-1}$ is compatible with $U_{p_{\mathrm{gen}}}({\beta_1})$.
Then for every $\beta_1 \ge \underline{\beta}_1$, the largest-stable-step strong recovery time satisfies
\[
  T_{\mathrm{str}}
  =
  \widetilde O
  \left(
    d\,
    \varepsilon^{-2}
    c_{\mathrm{wk}}^{-(p_{\mathrm{gen}}-1)}
    \beta_1^{2(I_*-1)}
  \right).
\]
\end{corollary}

\begin{proof}[Proof of Theorem~\ref{thm:feature-recovery-main}]
Let \(p=p_{\mathrm{gen}}\). By Lemma~\ref{lem:exponent_reduction}, for every
\(\beta_1\ge\underline{\beta}_1\),
\[
  |U_p(\beta_1)|
  \gtrsim
  \beta_1^{-(I_*-1)}.
\]
The student-side nondegeneracy assumption \(V_{p-1}\ne0\) ensures that the
corresponding student-side Hermite coefficient is nonzero. Under the random
initialization specified in Section~\ref{sec:problem-setup}, the standard
Gaussian SIM feature-learning guarantee for random spherical initialization
therefore applies to the reward-weighted teacher signal. It gives weak
recovery for a constant fraction of neurons with high probability, with
iteration complexity
\[
  \widetilde O\!\left(
    |U_p(\beta_1)|^{-2}
    d^{(p-1)\vee1}
  \right)
  =
  \widetilde O\!\left(
    \beta_1^{2(I_*-1)}
    d^{(p_{\mathrm{gen}}-1)\vee1}
  \right).
\]
Starting from weak recovery, the same feature-learning guarantee gives strong
recovery after an additional
\[
  \widetilde O\!\left(
    |U_p(\beta_1)|^{-2}d\varepsilon^{-2}
  \right)
  =
  \widetilde O\!\left(
    \beta_1^{2(I_*-1)}d\varepsilon^{-2}
  \right)
\]
iterations.
\end{proof}

\section{Low-temperature facts for the target tilted policy}
\label{app:one-dimensional-tilted-limit}

This appendix records the one-dimensional low-temperature facts used for the
temperature mismatch and projected-truncation bounds in Sections~\ref{subsec:temperature_mismatch_and_truncation} and~\ref{subsec:regret_bounds_projected_truncation}. Let
\[
  \phi(u)
  :=
  (2\pi)^{-1/2}e^{-u^2/2}
\]
be the standard Gaussian density on $\mathbb R$. For ${\beta_{2}}>0$, define
\[
  Z_{\beta_{2}}
  :=
  \int_{\mathbb R}e^{\sigma^*(u)/{\beta_{2}}}\phi(u)\,du,
  \qquad
  \mu_{\beta_{2}}(du)
  :=
  \frac{e^{\sigma^*(u)/{\beta_{2}}}\phi(u)}{Z_{\beta_{2}}}\,du.
\]

\begin{proposition}[Maximizers of an upper-bounded polynomial]
\label{prop:bounded-polynomial-maximizers}
Let $\sigma^*:\mathbb R\to\mathbb R$ be a nonconstant polynomial that is bounded above, and use the constant $B_*=\sup_{u\in\mathbb R}\sigma^*(u)$ from the setup. Then $\sigma^*$ attains its maximum, and the maximizer set
\[
  \mathcal M
  :=
  \{u\in\mathbb R:\sigma^*(u)=B_*\}
\]
is finite. Writing $\mathcal M=\{u_1,\dots,u_m\}$, for each $i\in\{1,\dots,m\}$ there exist $p_i\in\mathbb N$ and $c_i>0$ such that
\[
  \sigma^*(u_i+t)
  =
  B_*
  -
  c_i t^{2p_i}
  +
  o(t^{2p_i})
  \qquad
  (t\to 0).
\]
\end{proposition}

\begin{proof}
Since $\sigma^*$ is a nonconstant polynomial and is bounded above on $\mathbb R$, its degree is even and its leading coefficient is negative. Hence
\[
  \sigma^*(u)\to -\infty
  \qquad
  (|u|\to\infty).
\]
Therefore $\sigma^*$ attains its maximum on a sufficiently large compact interval. If $\mathcal M$ were infinite, then the nonzero polynomial $\sigma^*(u)-B_*$ would have infinitely many roots, which is impossible. Thus $\mathcal M$ is finite.

Fix $u_i\in\mathcal M$. The polynomial $\sigma^*(u_i+t)-B_*$ is not identically zero. Let $k_i$ be the order of its first nonzero Taylor coefficient at $t=0$. Since $u_i$ is a local maximum, this first nonzero order must be even and its coefficient must be negative. Thus $k_i=2p_i$ for some $p_i\in\mathbb N$, and the coefficient can be written as $-c_i$ with $c_i>0$.
\end{proof}

Define
\[
  p_{\max}
  :=
  \max_{1\le i\le m}p_i,
  \qquad
  I_{\max}
  :=
  \{i:\ p_i=p_{\max}\},
\]
and
\[
  \kappa
  :=
  \min\left\{
    \frac{1}{2p_{\max}},
    \;
    \min_{i\notin I_{\max}}
    \left(
      \frac{1}{2p_i}
      -
      \frac{1}{2p_{\max}}
    \right)
  \right\},
\]
where the second minimum is interpreted as $+\infty$ when
$I_{\max}=\{1,\dots,m\}$. Also define, for $i\in I_{\max}$,
\[
  A_i
  :=
  \phi(u_i)
  \int_{\mathbb R}e^{-c_i s^{2p_i}}\,ds.
\]

\begin{lemma}[Laplace asymptotics and zero-temperature weak limit]
\label{lem:tilted-marginal-weak-limit}
As ${\beta_{2}}\downarrow 0$,
\[
  Z_{\beta_{2}}
  =
  e^{B_*/{\beta_{2}}}
  {\beta_{2}}^{1/(2p_{\max})}
  \left(
    \sum_{i\in I_{\max}}A_i
    +
    o(1)
  \right).
\]
Moreover,
\[
  \mu_{\beta_{2}}
  \Rightarrow
  \mu_0
  :=
  \sum_{i\in I_{\max}}w_i\delta_{u_i},
  \qquad
  w_i
  :=
  \frac{A_i}{\sum_{j\in I_{\max}}A_j}.
\]
Equivalently, for every bounded continuous function $f:\mathbb R\to\mathbb R$,
\[
  \mathbb E_{U\sim\mu_{\beta_{2}}}[f(U)]
  \to
  \sum_{i\in I_{\max}}w_i f(u_i).
\]
\end{lemma}

\begin{proof}
For a bounded continuous function $f:\mathbb R\to\mathbb R$, define
\[
  N_{\beta_{2}}[f]
  :=
  \int_{\mathbb R}
  f(u)e^{\sigma^*(u)/{\beta_{2}}}\phi(u)\,du.
\]
Thus $Z_{\beta_{2}}=N_{\beta_{2}}[1]$.

Choose $\rho>0$ small enough that the intervals
\[
  (u_i-\rho,u_i+\rho),
  \qquad
  i=1,\dots,m,
\]
are disjoint. We also choose $\rho$ small enough that, for each $i$, the local expansion in Proposition~\ref{prop:bounded-polynomial-maximizers} implies
\[
  \sigma^*(u_i+t)
  \le
  B_*-\frac{c_i}{2}t^{2p_i}
  \qquad
  (|t|<\rho).
\]

First consider the contribution near a fixed maximizer $u_i$. Let
\[
  \alpha_i:=\frac{1}{2p_i},
  \qquad
  g_f(u):=f(u)\phi(u).
\]
After the change of variables $u=u_i+{\beta_{2}}^{\alpha_i}s$, we obtain
\[
\begin{aligned}
  &e^{-B_*/{\beta_{2}}}{\beta_{2}}^{-\alpha_i}
  \int_{u_i-\rho}^{u_i+\rho}
  f(u)e^{\sigma^*(u)/{\beta_{2}}}\phi(u)\,du \\
  &=
  \int_{|s|<\rho{\beta_{2}}^{-\alpha_i}}
  \exp\left(
    \frac{\sigma^*(u_i+{\beta_{2}}^{\alpha_i}s)-B_*}{{\beta_{2}}}
  \right)
  g_f(u_i+{\beta_{2}}^{\alpha_i}s)\,ds.
\end{aligned}
\]
For each fixed $s\in\mathbb R$, the integrand converges to
\[
  e^{-c_i s^{2p_i}}g_f(u_i).
\]
Moreover, for all sufficiently small ${\beta_{2}}$ the integrand is dominated by
\[
  \|g_f\|_{L^\infty(\mathbb R)}
  e^{-(c_i/2)s^{2p_i}},
\]
which is integrable over $\mathbb R$. By dominated convergence,
\[
  \int_{u_i-\rho}^{u_i+\rho}
  f(u)e^{\sigma^*(u)/{\beta_{2}}}\phi(u)\,du
  =
  e^{B_*/{\beta_{2}}}
  {\beta_{2}}^{1/(2p_i)}
  \left(
    f(u_i)\phi(u_i)
    \int_{\mathbb R}e^{-c_i s^{2p_i}}\,ds
    +
    o(1)
  \right).
\]

Next consider the complement of these neighborhoods. Since $\sigma^*(u)\to-\infty$ as $|u|\to\infty$ and all global maximizers have been removed, there exists $\eta_\rho>0$ such that
\[
  \sigma^*(u)\le B_*-\eta_\rho
\]
for all
\[
  u\notin\bigcup_{i=1}^m(u_i-\rho,u_i+\rho).
\]
Therefore
\[
  \left|
  \int_{\mathbb R\setminus\cup_{i=1}^m(u_i-\rho,u_i+\rho)}
  f(u)e^{\sigma^*(u)/{\beta_{2}}}\phi(u)\,du
  \right|
  \le
  \|f\|_{L^\infty(\mathbb R)}e^{(B_*-\eta_\rho)/{\beta_{2}}},
\]
which is
\[
  o\left(
    e^{B_*/{\beta_{2}}}{\beta_{2}}^{1/(2p_{\max})}
  \right).
\]

If $p_i<p_{\max}$, then
\[
  {\beta_{2}}^{1/(2p_i)}
  =
  o\left({\beta_{2}}^{1/(2p_{\max})}\right).
\]
Combining the local contributions and the complement estimate gives
\[
  N_{\beta_{2}}[f]
  =
  e^{B_*/{\beta_{2}}}
  {\beta_{2}}^{1/(2p_{\max})}
  \left(
    \sum_{i\in I_{\max}}f(u_i)A_i
    +
    o(1)
  \right).
\]
Taking $f\equiv 1$ gives the stated asymptotic for $Z_{\beta_{2}}$. Since $A_i>0$ for each $i\in I_{\max}$, division by the asymptotic for $Z_{\beta_{2}}$ yields
\[
  \mathbb E_{U\sim\mu_{\beta_{2}}}[f(U)]
  =
  \frac{N_{\beta_{2}}[f]}{Z_{\beta_{2}}}
  \to
  \frac{\sum_{i\in I_{\max}}f(u_i)A_i}{\sum_{j\in I_{\max}}A_j}
  =
  \sum_{i\in I_{\max}}w_i f(u_i).
\]
This is exactly weak convergence of $\mu_{\beta_{2}}$ to $\mu_0$.
\end{proof}

\begin{lemma}[Log-partition derivatives and low-temperature variance]
\label{lem:log-partition-variance-appendix}
Define, for $\lambda>0$,
\[
  \mathcal Z(\lambda)
  :=
  \int_{\mathbb R}e^{\lambda\sigma^*(u)}\phi(u)\,du,
  \qquad
  \psi(\lambda):=\log\mathcal Z(\lambda).
\]
Then
\[
  \psi'(\lambda)
  =
  \mathbb E_{U\sim\mu_{1/\lambda}}[\sigma^*(U)],
  \qquad
  \psi''(\lambda)
  =
  \operatorname{Var}_{U\sim\mu_{1/\lambda}}[\sigma^*(U)].
\]
Moreover, as $\lambda\to\infty$,
\[
  \psi''(\lambda)
  =
  \frac{1}{2p_{\max}}\lambda^{-2}
  +
  O(\lambda^{-2-\kappa}),
\]
where $\kappa$ is defined above from the local orders of the dominant
maximizers of $\sigma^*$.
\end{lemma}

\begin{proof}
Differentiation under the integral is justified by the upper bound $\sigma^*(u)\le B_*$ and the polynomial growth of $\sigma^*$. This gives
\[
  \psi'(\lambda)
  =
  \frac{
    \int_{\mathbb R}\sigma^*(u)e^{\lambda\sigma^*(u)}\phi(u)\,du
  }{
    \int_{\mathbb R}e^{\lambda\sigma^*(u)}\phi(u)\,du
  }
  =
  \mathbb E_{U\sim\mu_{1/\lambda}}[\sigma^*(U)]
\]
and the usual exponential-family identity
\[
  \psi''(\lambda)
  =
  \operatorname{Var}_{U\sim\mu_{1/\lambda}}[\sigma^*(U)].
\]

It remains to identify the leading low-temperature behavior of this variance. Write
\[
  D(u):=B_*-\sigma^*(u)\ge 0.
\]
Since adding the constant $B_*$ does not change variance,
\[
  \operatorname{Var}_{U\sim\mu_{1/\lambda}}[\sigma^*(U)]
  =
  \operatorname{Var}_{U\sim\mu_{1/\lambda}}[D(U)].
\]
The same localization argument used in Lemma~\ref{lem:tilted-marginal-weak-limit}, now applied to the weighted integrals with factors $D(u)$ and $D(u)^2$, gives
\[
  \mathbb E_{U\sim\mu_{1/\lambda}}[D(U)]
  =
  \frac{1}{2p_{\max}}\lambda^{-1}
  +
  O(\lambda^{-1-\kappa}),
\]
and
\[
  \mathbb E_{U\sim\mu_{1/\lambda}}[D(U)^2]
  =
  \left(
    \frac{1}{2p_{\max}}
    +
    \frac{1}{4p_{\max}^2}
  \right)
  \lambda^{-2}
  +
  O(\lambda^{-2-\kappa}).
\]
Indeed, near each dominant maximizer $u_i$ with $i\in I_{\max}$, the change of variables
\[
  u=u_i+\lambda^{-1/(2p_{\max})}s
\]
turns $D(u)$ into $\lambda^{-1}c_i s^{2p_{\max}}+o(\lambda^{-1})$. Under the limiting local density proportional to $e^{-c_i s^{2p_{\max}}}$, the random variable $c_i s^{2p_{\max}}$ has first two moments
\[
  \frac{1}{2p_{\max}},
  \qquad
  \frac{1}{2p_{\max}}
  \left(
    \frac{1}{2p_{\max}}+1
  \right),
\]
respectively. The non-dominant maximizers and the complement of the maximizer neighborhoods contribute only to the stated remainder order. Hence
\[
\begin{aligned}
  \operatorname{Var}_{U\sim\mu_{1/\lambda}}[D(U)]
  &=
  \mathbb E_{U\sim\mu_{1/\lambda}}[D(U)^2]
  -
  \left(
    \mathbb E_{U\sim\mu_{1/\lambda}}[D(U)]
  \right)^2 \\
  &=
  \frac{1}{2p_{\max}}\lambda^{-2}
  +
  O(\lambda^{-2-\kappa}).
\end{aligned}
\]
This proves the asserted expansion for $\psi''(\lambda)$.
\end{proof}

\begin{lemma}[Low-temperature variance asymptotics]
\label{lem:low-temp-variance}
As ${\beta_{2}}\to 0$,
\[
  \operatorname{Var}_{x\sim\pi_{\beta_{2}}}[r^*(x)]
  =
  \frac{{\beta_{2}}^2}{2p_{\max}}
  +
  O({\beta_{2}}^{2+\kappa}).
\]
\end{lemma}

\begin{proof}[Proof of Lemma~\ref{lem:low-temp-variance}]
Set $\lambda=1/{\beta_{2}}$. Since $r^*(x)=\sigma^*(\langle\theta^*,x\rangle)$ and $U=\langle\theta^*,x\rangle$ has marginal distribution $\mu_{\beta_{2}}$ under $x\sim\pi_{\beta_{2}}$,
\[
  \operatorname{Var}_{x\sim\pi_{\beta_{2}}}[r^*(x)]
  =
  \operatorname{Var}_{U\sim\mu_{\beta_{2}}}[\sigma^*(U)]
  =
  \psi''(1/{\beta_{2}}).
\]
Lemma~\ref{lem:log-partition-variance-appendix} therefore gives
\[
  \operatorname{Var}_{x\sim\pi_{\beta_{2}}}[r^*(x)]
  =
  \frac{{\beta_{2}}^2}{2p_{\max}}
  +
  O({\beta_{2}}^{2+\kappa}),
\]
as claimed.
\end{proof}

\section{Temperature-mismatch and projected-truncation bounds}
\label{app:proof-temperature-mismatch}

Building on the low-temperature variance expansion in Appendix~\ref{app:one-dimensional-tilted-limit}, this appendix proves the two common non-learning controls used in the value-gap theorems: the temperature-mismatch bound and the projected-truncation bound.

\begin{proof}[Proof of the temperature-mismatch part of Proposition~\ref{prop:low-temp-temperature-mismatch}]
By the identity for $T_{\mathrm{temp}}$ in Section~\ref{subsec:temperature_mismatch_and_truncation},
\[
  T_{\mathrm{temp}}
  =
  \int_{1/{\beta_{2}}}^{1/\beta^*}
  \operatorname{Var}_{x\sim\pi_{\beta'}}[r^*(x)]
  \,d\left(\frac{1}{\beta'}\right).
\]
By Lemma~\ref{lem:low-temp-variance}, after choosing ${\overline{\beta}_{2}}>0$ sufficiently small, there exists $C_{\mathrm{temp}}>0$ such that
\[
  \operatorname{Var}_{x\sim\pi_{\beta'}}[r^*(x)]
  \le
  \left(
    \frac{1}{2p_{\max}}
    +
    C_{\mathrm{temp}}{\overline{\beta}_{2}}^\kappa
  \right)
  (\beta')^2
\]
for all $0<\beta'\le{\overline{\beta}_{2}}$. Hence, for all $0<{\beta_{2}},\beta^*\le{\overline{\beta}_{2}}$,
\[
\begin{aligned}
  |T_{\mathrm{temp}}|
  &\le
  \left(
    \frac{1}{2p_{\max}}
    +
    C_{\mathrm{temp}}{\overline{\beta}_{2}}^\kappa
  \right)
  \left|
    \int_{1/{\beta_{2}}}^{1/\beta^*}
    s^{-2}
    \,ds
  \right| \\
  &=
  \left(
    \frac{1}{2p_{\max}}
    +
    C_{\mathrm{temp}}{\overline{\beta}_{2}}^\kappa
  \right)
  |\beta^*-{\beta_{2}}|.
\end{aligned}
\]
\end{proof}

\begin{lemma}[Uniform projected-truncation bound for the tilted tail probability]
\label{lem:tilted-tail-saturation}
For $R>0$ and ${\beta_{2}}>0$, define
\[
  \alpha_{{\beta_{2}},R}
  :=
  \pi_{\beta_{2}}(B_R^c).
\]
If
\[
  R>\max_{i\in I_{\max}}|u_i|,
\]
then
\[
  \overline\alpha_R
  :=
  \sup_{{\beta_{2}}>0}\alpha_{{\beta_{2}},R}
  <
  1.
\]
\end{lemma}

\begin{proof}
Let
\[
  S_0:=S\cap(\theta^*)^\perp .
\]
Write the projected coordinates as
\[
  P_Sx=U\theta^*+Z,
  \qquad
  U:=\langle\theta^*,x\rangle,
  \qquad
  Z\in S_0 .
\]
Under \(x\sim\pi_{\beta_{2}}\), the coordinate \(U\) has the tilted
one-dimensional marginal \(\mu_{\beta_{2}}\), while
\(Z\sim\gamma_{S_0}\) is independent of \(U\). Since
\[
  B_R^c=\{x:\|P_Sx\|_2>R\},
\]
we have
\[
  B_R^c
  =
  \{U^2+\|Z\|_2^2>R^2\}.
\]
Define
\[
  g_R(u)
  :=
  \mathbb P_{Z\sim\gamma_{S_0}}
  \left[
    u^2+\|Z\|_2^2>R^2
  \right].
\]
Then \(g_R\) is bounded and continuous at each \(u_i\in I_{\max}\), and
\[
  \alpha_{{\beta_{2}},R}
  =
  \mathbb E_{U\sim\mu_{\beta_{2}}}[g_R(U)].
\]
By Lemma~\ref{lem:tilted-marginal-weak-limit},
\[
  \alpha_{{\beta_{2}},R}
  \to
  \alpha_{0,R}
  :=
  \sum_{i\in I_{\max}}w_i g_R(u_i).
\]
If \(R>\max_{i\in I_{\max}}|u_i|\), then \(g_R(u_i)<1\) for every
\(i\in I_{\max}\), and therefore \(\alpha_{0,R}<1\).

For every fixed \({\beta_{2}}>0\), we have
\(\alpha_{{\beta_{2}},R}<1\) because \(\pi_{\beta_{2}}\) has a strictly
positive density with respect to \(\pi_{\mathrm{ref}}\) and
\(\pi_{\mathrm{ref}}(B_R)>0\). To turn this pointwise strict inequality
into a uniform one, we control the two endpoints
\({\beta_{2}}\downarrow0\) and \({\beta_{2}}\to\infty\), and then use
compactness on the remaining middle range.

The map
\[
  {\beta_{2}}\mapsto \alpha_{{\beta_{2}},R}
\]
is continuous on \((0,\infty)\) because both
\[
  \mathbb E_{x\sim\pi_{\mathrm{ref}}}
  \left[
    \mathbf 1_{B_R^c}(x)e^{r^*(x)/{\beta_{2}}}
  \right]
  \quad\text{and}\quad
  \mathbb E_{x\sim\pi_{\mathrm{ref}}}
  \left[
    e^{r^*(x)/{\beta_{2}}}
  \right]
\]
are continuous in \({\beta_{2}}>0\) by dominated convergence, using the
upper bound \(r^*(x)\le B_*\). The denominator is strictly positive, so
the ratio defining \(\alpha_{{\beta_{2}},R}\) is continuous on
\((0,\infty)\).

As \({\beta_{2}}\to\infty\), dominated convergence gives
\[
  \alpha_{{\beta_{2}},R}
  \to
  \pi_{\mathrm{ref}}(B_R^c)
  <
  1.
\]
Together with
\(\alpha_{{\beta_{2}},R}\to\alpha_{0,R}<1\) as
\({\beta_{2}}\downarrow0\), this implies that there exist
\(0<\beta_-<\beta_+<\infty\) and \(\eta>0\) such that
\[
  \alpha_{{\beta_{2}},R}\le 1-\eta
\]
whenever
\({\beta_{2}}\in(0,\beta_-]\cup[\beta_+,\infty)\).

On the compact interval \([\beta_-,\beta_+]\), continuity implies that
\(\alpha_{{\beta_{2}},R}\) attains its maximum. Since
\(\alpha_{{\beta_{2}},R}<1\) for every fixed \({\beta_{2}}>0\), this
maximum is also strictly smaller than \(1\). Combining the three regions
gives
\[
  \sup_{{\beta_{2}}>0}\alpha_{{\beta_{2}},R}<1.
\]
\end{proof}

\begin{proof}[Proof of the projected-truncation part of  Proposition~\ref{prop:low-temp-truncation}]
Let
\[
  \alpha_{{\beta_{2}},R}:=\pi_{\beta_{2}}(B_R^c).
\]
By the identity for $T_{\mathrm{cut}}(R)$ in Section~\ref{subsec:temperature_mismatch_and_truncation},
\[
  T_{\mathrm{cut}}(R)
  =
  \frac{1}{1-\alpha_{{\beta_{2}},R}}
  \mathbb E_{x\sim\pi_{\beta_{2}}}
  \left[
    (r^*(x)-m_{\beta_{2}})\mathbf 1_{B_R^c}(x)
  \right].
\]
By Cauchy--Schwarz,
\[
  |T_{\mathrm{cut}}(R)|
  \le
  \frac{
    \sqrt{\operatorname{Var}_{x\sim\pi_{\beta_{2}}}[r^*(x)]\,\alpha_{{\beta_{2}},R}}
  }{
    1-\alpha_{{\beta_{2}},R}
  }.
\]
Fix a sufficiently small ${\overline{\beta}_{2}}>0$. Since $R$ is taken sufficiently large,  $R>\max_{i\in I_{\max}}|u_i|$. Then Lemma~\ref{lem:tilted-tail-saturation} gives $\overline\alpha_R<1$ such that $\alpha_{{\beta_{2}},R}\le \overline\alpha_R$ for all ${\beta_{2}}>0$. Lemma~\ref{lem:low-temp-variance} implies, after reducing the allowed size of ${\overline{\beta}_{2}}$ if necessary, that
\[
  \operatorname{Var}_{x\sim\pi_{\beta_{2}}}[r^*(x)]
  \le
  C_{\mathrm{var},{\overline{\beta}_{2}}}{\beta_{2}}^2
\]
for all $0<{\beta_{2}}\le{\overline{\beta}_{2}}$. Therefore
\[
  |T_{\mathrm{cut}}(R)|
  \le
  \frac{\sqrt{\overline\alpha_R C_{\mathrm{var},{\overline{\beta}_{2}}}}}{1-\overline\alpha_R}\,{\beta_{2}}.
\]
Absorbing the prefactor into $C_{R,{\overline{\beta}_{2}}}$ proves the claim.
\end{proof}

\section{From weighted prediction to tilted-policy value}
\label{app:three-factor-learning-bridge}

This appendix explains why a weighted prediction bound is sufficient for controlling the tilted-policy value gap. 
The key point is that, along the interpolation path from the true reward to the fitted reward, differentiating the tilted expectation produces a covariance between the true reward and the reward error. 
The density-ratio factor in Lemma~\ref{lem:learning-bridge} is the cost of controlling this covariance by an $L^2$ prediction error under the weighted regression measure. 
We first record the label-weighted population shift, then prove a general three-factor bridge inequality, and finally apply its simplified $L^2$ form to prove Lemma~\ref{lem:learning-bridge}.

\subsection{Label-weighted population shift}
\label{app:label-shift-calculation}

For each fixed \(x\in B_R\),
\[
\begin{aligned}
&
\mathbb E_{\zeta\sim\mathrm{Unif}[-\tau,\tau]}
\left[
  e^{(r^*(x)+\zeta)/\beta_2}
  \left(r^*(x)+\zeta-f(x)\right)^2
\right]
\\
&\qquad =
e^{r^*(x)/\beta_2}Z_\zeta(\beta_2)
\left\{
  \left(f(x)-r^*(x)-m_{\zeta,\beta_2}\right)^2
  +
  \operatorname{Var}_{\zeta\sim\nu_{\zeta,\beta_2}}[\zeta]
\right\},
\end{aligned}
\]
where
\[
  Z_\zeta(\beta_2)
  :=
  \mathbb E_{\zeta\sim\mathrm{Unif}[-\tau,\tau]}
  \left[e^{\zeta/\beta_2}\right],
  \qquad
  \nu_{\zeta,\beta_2}(d\zeta)
  :=
  \frac{e^{\zeta/\beta_2}}{Z_\zeta(\beta_2)}
  \mathrm{Unif}[-\tau,\tau](d\zeta),
\]
and
\[
  m_{\zeta,\beta_2}
  :=
  \mathbb E_{\zeta\sim\nu_{\zeta,\beta_2}}[\zeta].
\]
Thus the label-weighted population risk targets
\(r^*+m_{\zeta,\beta_2}\), an additive shift of the true reward.

\begin{lemma}[Three-factor bridge for the learning term]
\label{lem:three-factor-learning-bridge}
Let $\nu_R$ be a probability measure supported on $B_R$, and let $(\pi_t)_{t\in[0,1]}$ be probability measures supported on $B_R$ such that $\pi_t\ll\nu_R$ for every $t\in[0,1]$. Let $\widehat r_R:B_R\to\mathbb R$ be any learned reward and define
\[
  Y_t(x)
  :=
  r^*(x)
  -
  \mathbb E_{z\sim\pi_t}[r^*(z)].
\]
For H\"older exponents $p_{\mathrm{Holder}}\in(2,\infty]$ and $q_{\mathrm{Holder}}\in[2,\infty)$ satisfying
\[
  \frac{1}{p_{\mathrm{Holder}}}+\frac{1}{q_{\mathrm{Holder}}}=\frac{1}{2},
\]
set
\[
  \mathcal A_{p_{\mathrm{Holder}},q_{\mathrm{Holder}},R}[\nu_R,\pi_\cdot]
  :=
  \int_0^1
  \|Y_t\|_{L^{q_{\mathrm{Holder}}}(\pi_t)}
  \left\|
    \left(
      \frac{d\pi_t}{d\nu_R}
    \right)^{1-1/q_{\mathrm{Holder}}}
  \right\|_{L^{p_{\mathrm{Holder}}}(\nu_R)}
  \,dt.
\]
Then
\[
  \left|
    \int_0^1
    \mathbb E_{x\sim\pi_t}
    \left[
      Y_t(x)(\widehat r_R(x)-r^*(x))
    \right]
    \,dt
  \right|
  \le
  \mathcal A_{p_{\mathrm{Holder}},q_{\mathrm{Holder}},R}[\nu_R,\pi_\cdot]
  \inf_{c\in\mathbb R}
  \|\widehat r_R-r^*-c\|_{L^2(\nu_R)}.
\]
The following simplified form follows from the same argument. If
\[
  M_R:=\sup_{x\in B_R}|r^*(x)|
\]
and
\[
  \mathcal D_R[\nu_R,\pi_\cdot]
  :=
  \int_0^1
  \left\|
    \frac{d\pi_t}{d\nu_R}
  \right\|_{L^2(\nu_R)}
  \,dt,
\]
then
\begin{equation}
\label{eq:simplified-three-factor-bridge}
  \left|
    \int_0^1
    \mathbb E_{x\sim\pi_t}
    \left[
      Y_t(x)(\widehat r_R(x)-r^*(x))
    \right]
    \,dt
  \right|
  \le
  2M_R\mathcal D_R[\nu_R,\pi_\cdot]
  \inf_{c\in\mathbb R}
  \|\widehat r_R-r^*-c\|_{L^2(\nu_R)}.
\end{equation}
\end{lemma}

\begin{proof}
Fix $c\in\mathbb R$ and write
\[
  h_c(x):=\widehat r_R(x)-r^*(x)-c,
  \qquad
  \rho_t(x):=\frac{d\pi_t}{d\nu_R}(x).
\]
Since
\[
  \mathbb E_{x\sim\pi_t}[Y_t(x)]=0,
\]
the left-hand side is unchanged if $\widehat r_R-r^*$ is replaced by $h_c$. For each $t\in[0,1]$,
\[
\begin{aligned}
  \left|
  \mathbb E_{x\sim\pi_t}[Y_t(x)h_c(x)]
  \right|
  &=
  \left|
  \mathbb E_{x\sim\nu_R}
  \left[
    Y_t(x)h_c(x)\rho_t(x)
  \right]
  \right| \\
  &=
  \left|
  \mathbb E_{x\sim\nu_R}
  \left[
    \bigl(Y_t(x)\rho_t(x)^{1/q_{\mathrm{Holder}}}\bigr)
    h_c(x)
    \rho_t(x)^{1-1/q_{\mathrm{Holder}}}
  \right]
  \right| \\
  &\le
  \|Y_t\|_{L^{q_{\mathrm{Holder}}}(\pi_t)}
  \|h_c\|_{L^2(\nu_R)}
  \left\|
    \rho_t^{1-1/q_{\mathrm{Holder}}}
  \right\|_{L^{p_{\mathrm{Holder}}}(\nu_R)}.
\end{aligned}
\]
Here the last line is the three-factor H\"older inequality under $x\sim\nu_R$, using $1/q_{\mathrm{Holder}}+1/2+1/p_{\mathrm{Holder}}=1$. Integrating over $t$ and then taking the infimum over $c\in\mathbb R$ proves the first claim. It remains to prove \eqref{eq:simplified-three-factor-bridge}. Use
\[
\begin{aligned}
  \left|
  \mathbb E_{x\sim\pi_t}[Y_t(x)h_c(x)]
  \right|
  &\le
  \sup_{x\in B_R}|Y_t(x)|
  \mathbb E_{x\sim\nu_R}[|h_c(x)|\rho_t(x)] \\
  &\le
  2M_R
  \|h_c\|_{L^2(\nu_R)}
  \|\rho_t\|_{L^2(\nu_R)},
\end{aligned}
\]
where
\[
  \sup_{x\in B_R}|Y_t(x)|
  \le
  2M_R,
\]
because $\pi_t$ is supported on $B_R$.
\end{proof}

We now apply the simplified \(L^2\) form of the preceding lemma to the
interpolation paths defined in Subsection~\ref{subsec:learned_term_bound}.
This gives the bridge inequality stated in Lemma~\ref{lem:learning-bridge}.

\begin{proof}[Proof of Lemma~\ref{lem:learning-bridge}]
Fix \(\mathrm w\in\{\mathrm{lbl},\mathrm{surr}\}\), and write
\[
  \pi_t:=\pi_{{\beta_{2}},t,R}^{\mathrm w},
  \qquad
  \nu_R:=\nu_R^{\mathrm w},
  \qquad
  \widehat r_R:=\widehat r_R^{\mathrm w}.
\]
By differentiating the interpolation path,
\[
  T_{\mathrm{learn}}^{\mathrm w}(R)
  =
  -\frac{1}{{\beta_{2}}}
  \int_0^1
  \mathbb E_{x\sim\pi_t}
  \left[
    \left(
      r^*(x)
      -
      \mathbb E_{z\sim\pi_t}[r^*(z)]
    \right)
    (\widehat r_R(x)-r^*(x))
  \right]
  dt.
\]
Thus
\[
  |T_{\mathrm{learn}}^{\mathrm w}(R)|
  \le
  \frac{1}{{\beta_{2}}}
  \left|
  \int_0^1
  \mathbb E_{x\sim\pi_t}
  \left[
    Y_t(x)(\widehat r_R(x)-r^*(x))
  \right]
  dt
  \right|,
\]
where
\[
  Y_t(x)
  :=
  r^*(x)
  -
  \mathbb E_{z\sim\pi_t}[r^*(z)].
\]
Applying the simplified bound
\eqref{eq:simplified-three-factor-bridge} from
Lemma~\ref{lem:three-factor-learning-bridge}, with
\[
  \mathcal D_R[\nu_R,\pi_\cdot]
  =
  \mathcal D_{\mathrm w,R},
\]
gives
\[
  |T_{\mathrm{learn}}^{\mathrm w}(R)|
  \le
  \frac{2M_R\mathcal D_{\mathrm w,R}}{{\beta_{2}}}
  \inf_{c\in\mathbb R}
  \|\widehat r_R-r^*-c\|_{L^2(\nu_R)}.
\]
This is the desired bound.
\end{proof}

\section{Weighted ridge estimates after feature recovery}
\label{app:weighted-ridge-learning}

This appendix proves the finite-sample weighted ridge estimates used by the
label-weighted and surrogate-weighted learning bounds. For brevity, write
\(\hat r_R:=\widehat r_R^{\mathrm{lbl}}\) and
\(\hat r_R^{(0)}:=\widehat r_R^{\mathrm{surr}}\) in this appendix. For the vector notation used in this appendix, define
\[
  \psi_N(x)
  :=
  \frac{1}{N}
  \left(\sigma(\langle w_j,x\rangle+b_j)\right)_{j=1}^N
\]
so that $r_a(x)=\langle a,\psi_N(x)\rangle$. We first collect the weighted
moment notation used in the exact bounds.
For label weighting, define
\[
  Z_\zeta
  :=
  \mathbb E_{\zeta\sim\mathrm{Unif}[-\tau,\tau]}
  \left[
    e^{\zeta/{\beta_{2}}}
  \right],
  \qquad
  M_{*,R}:=\sup_{x\in B_R}r^*(x),
\]
and
\[
  m_{\zeta,{\beta_{2}}}
  :=
  \mathbb E_{\zeta\sim\nu_{\zeta,{\beta_{2}}}}[\zeta],
  \qquad
  \nu_{\zeta,{\beta_{2}}}(d\zeta)
  :=
  \frac{e^{\zeta/{\beta_{2}}}}{Z_\zeta}
  \mathrm{Unif}[-\tau,\tau](d\zeta).
\]
Set
\[
  r_{\mathrm{lbl},{\beta_{2}}}^*(x):=r^*(x)+m_{\zeta,{\beta_{2}}}.
\]
The unnormalized label-weighted moments are
\[
  M_{2,\mathrm{wt},R}
  :=
  \mathbb E_{x\sim\pi_{\mathrm{ref}},\,\zeta\sim\mathrm{Unif}[-\tau,\tau]}
  \left[
    \mathbf 1_{B_R}(x)e^{2(r^*(x)+\zeta)/{\beta_{2}}}
    (\zeta-m_{\zeta,{\beta_{2}}})^2
    \|\psi_N(x)\|_2^2
  \right],
\]
\[
  M_{4,\mathrm{wt},R}
  :=
  \mathbb E_{x\sim\pi_{\mathrm{ref}},\,\zeta\sim\mathrm{Unif}[-\tau,\tau]}
  \left[
    \mathbf 1_{B_R}(x)e^{2(r^*(x)+\zeta)/{\beta_{2}}}
    \|\psi_N(x)\|_2^4
  \right],
\]
and
\[
  G_{4,\mathrm{wt},R}
  :=
  \mathbb E_{x\sim\pi_{\mathrm{ref}},\,\zeta\sim\mathrm{Unif}[-\tau,\tau]}
  \left[
    \mathbf 1_{B_R}(x)e^{2(r^*(x)+\zeta)/{\beta_{2}}}
    |r_{\mathrm{lbl},{\beta_{2}}}^*(x)|^4
  \right].
\]
For surrogate weighting, define
\[
  M_{2,\mathrm{wt},R}^{(0)}
  :=
  \mathbb E_{x\sim\pi_{\mathrm{ref}},\,\zeta\sim\mathrm{Unif}[-\tau,\tau]}
  \left[
    \mathbf 1_{B_R}(x)e^{2r_{a_0}(x)/{\beta_{2}}}
    \zeta^2
    \|\psi_N(x)\|_2^2
  \right],
\]
\[
  M_{4,\mathrm{wt},R}^{(0)}
  :=
  \mathbb E_{x\sim\pi_{\mathrm{ref}},\,\zeta\sim\mathrm{Unif}[-\tau,\tau]}
  \left[
    \mathbf 1_{B_R}(x)e^{2r_{a_0}(x)/{\beta_{2}}}
    \|\psi_N(x)\|_2^4
  \right],
\]
and
\[
  G_{4,\mathrm{wt},R}^{(0)}
  :=
  \mathbb E_{x\sim\pi_{\mathrm{ref}},\,\zeta\sim\mathrm{Unif}[-\tau,\tau]}
  \left[
    \mathbf 1_{B_R}(x)e^{2r_{a_0}(x)/{\beta_{2}}}
    |r^*(x)|^4
  \right].
\]

\begin{lemma}[Second-layer comparators from recovered first-layer features]
\label{lem:second-layer-comparators-from-recovered-features}
Work under the first-layer recovery condition from the second-layer weighted ridge setup.
Let $P_{T_2}=T_2^{-1}\sum_{i=1}^{T_2}\delta_{x_i}$ be the empirical distribution of the ridge-fitting inputs. Then, for each of the two targets
\[
  r_{\mathrm{target}}\in\{r^*,\ r_{\mathrm{lbl},{\beta_{2}}}^*\},
  \qquad
  r_{\mathrm{lbl},{\beta_{2}}}^*(x):=r^*(x)+m_{\zeta,{\beta_{2}}},
\]
there exists a comparator $a^\sharp\in\mathbb R^N$ such that
\[
  \|r_{a^\sharp}-r_{\mathrm{target}}\|_{L^2(P_{T_2})}^2
  \lesssim
  N^{-2}+\varepsilon^2,
  \qquad
  \|a^\sharp\|_2^2
  \lesssim
  N.
\]
The implicit constants may depend on the fixed reward and activation classes and on $\tau$, but not on $N,T_2,\lambda,\delta_0$.
\end{lemma}

\begin{proof}
Apply \citet[Appendix B.6, Lemma 22]{lee2024neural} under the first-layer recovery condition from the second-layer weighted ridge setup. This lemma gives a low-norm second-layer approximation for finite-degree ridge-polynomial targets. Both $r^*$ and $r_{\mathrm{lbl},{\beta_{2}}}^*=r^*+m_{\zeta,{\beta_{2}}}$ belong to this class; the latter only changes the constant coefficient. Therefore the cited construction gives the stated comparators for both targets.
\end{proof}

For the abstract weighted ridge estimate, fix a measurable weight $W_R(x,\zeta)\ge 0$ supported on $B_R$, a regression target $r_{\mathrm{target}}:B_R\to\mathbb R$, and a residual $\xi_R(x,\zeta)$ such that the observed label can be written as
\[
  y=r_{\mathrm{target}}(x)+\xi_R(x,\zeta).
\]
We use the ridge-fitting sample $(x_i,\zeta_i,y_i)_{i=1}^{T_2}$ defined in the Section~\ref{subsec:second-layer-weighted-ridge-setup}.
Assume
\[
  0<Z_W
  :=
  \mathbb E_{x\sim\pi_{\mathrm{ref}},\,\zeta\sim\mathrm{Unif}[-\tau,\tau]}
  [W_R(x,\zeta)]
  <\infty.
\]
Define the probability measure $\nu_W$ on $B_R$ by
\[
  \nu_W(dx)
  :=
  \frac{
    \mathbb E_{\zeta\sim\mathrm{Unif}[-\tau,\tau]}[W_R(x,\zeta)]
  }{Z_W}
  \pi_{\mathrm{ref}}(dx).
\]
Define the unnormalized weighted moments
\[
  M_{2,W,R}
  :=
  \mathbb E_{x\sim\pi_{\mathrm{ref}},\,\zeta\sim\mathrm{Unif}[-\tau,\tau]}
  \left[
    W_R(x,\zeta)^2\xi_R(x,\zeta)^2\|\psi_N(x)\|_2^2
  \right],
\]
\[
  M_{4,W,R}
  :=
  \mathbb E_{x\sim\pi_{\mathrm{ref}},\,\zeta\sim\mathrm{Unif}[-\tau,\tau]}
  \left[
    W_R(x,\zeta)^2\|\psi_N(x)\|_2^4
  \right],
\]
and
\[
  G_{4,W,R}
  :=
  \mathbb E_{x\sim\pi_{\mathrm{ref}},\,\zeta\sim\mathrm{Unif}[-\tau,\tau]}
  \left[
    W_R(x,\zeta)^2|r_{\mathrm{target}}(x)|^4
  \right].
\]
For a measurable $h:B_R\to\mathbb R$, write
\[
  \|h\|_{W,T_2}^2
  :=
  \frac{1}{T_2}\sum_{i=1}^{T_2}
  W_R(x_i,\zeta_i)h(x_i)^2,
  \qquad
  \|h\|_{W,\mathrm{pop}}^2
  :=
  \mathbb E_{x\sim\pi_{\mathrm{ref}},\,\zeta\sim\mathrm{Unif}[-\tau,\tau]}
  \left[
    W_R(x,\zeta)h(x)^2
  \right].
\]

\begin{lemma}[Abstract truncated weighted ridge prediction bound]
\label{lem:abstract-truncated-weighted-ridge}
Assume the following weighted residual orthogonality:
\[
  \mathbb E_{x\sim\pi_{\mathrm{ref}},\,\zeta\sim\mathrm{Unif}[-\tau,\tau]}
  \left[
    W_R(x,\zeta)\xi_R(x,\zeta)\psi_N(x)
  \right]
  =
  0.
\]
Let
\[
  \hat a_W
  \in
  \arg\min_{a\in\mathbb R^N}
  \left\{
    \frac{1}{T_2}\sum_{i=1}^{T_2}
    W_R(x_i,\zeta_i)(y_i-r_a(x_i))^2
    +
    \lambda\|a\|_2^2
  \right\}.
\]
For any comparator $a^\sharp\in\mathbb R^N$, set
\[
  \mathcal E_{W,T_2}(a^\sharp)
  :=
  \|r_{a^\sharp}-r_{\mathrm{target}}\|_{W,T_2}^2.
\]
There exists a universal constant $C>0$ such that, with probability at least $1-4\delta_0$ over the ridge-fitting sample, if
\[
  \lambda
  \ge
  C\sqrt{\frac{M_{4,W,R}}{T_2\delta_0}},
\]
then
\[
  \|r_{\hat a_W}-r_{\mathrm{target}}\|_{L^2(\nu_W)}^2
  \lesssim
  \frac{1}{Z_W}
  \left(
    \mathcal E_{W,T_2}(a^\sharp)
    +
    \frac{
      M_{2,W,R}
      +
      M_{4,W,R}
      +
      G_{4,W,R}
    }{
      T_2\lambda\delta_0
    }
    +
    \sqrt{\frac{G_{4,W,R}}{T_2\delta_0}}
    +
    \lambda\|a^\sharp\|_2^2
  \right).
\]
\end{lemma}

\begin{proof}
The argument is the weighted analogue of the second-layer ridge proof of \citet[Appendix B.6, Lemma 20]{lee2024neural}; the same unweighted step is also used as a black-box prediction-error argument in \citet[Proposition B.10]{tsiolis2025information}.
In this proof, write $\mathbb E_{\mathrm{sample}}$ for expectation over
\[
  (x_i,\zeta_i)_{i=1}^{T_2}
  \sim
  \left(
    \pi_{\mathrm{ref}}\otimes\mathrm{Unif}[-\tau,\tau]
  \right)^{T_2}.
\]

First, ridge optimality compares $\hat a_W$ with $a^\sharp$. Since $y=r_{\mathrm{target}}+\xi_R$, expanding the two weighted empirical objectives gives
\[
  \|r_{\hat a_W}-r_{\mathrm{target}}\|_{W,T_2}^2
  +
  \lambda\|\hat a_W\|_2^2
  \le
  \mathcal E_{W,T_2}(a^\sharp)
  +
  2\langle a^\sharp-\hat a_W,V\rangle
  +
  \lambda\|a^\sharp\|_2^2,
\]
where
\[
  V
  :=
  \frac{1}{T_2}\sum_{i=1}^{T_2}
  W_R(x_i,\zeta_i)\xi_R(x_i,\zeta_i)\psi_N(x_i).
\]
The weighted residual orthogonality implies
\[
  \mathbb E_{\mathrm{sample}}[V]=0,
  \qquad
  \mathbb E_{\mathrm{sample}}[\|V\|_2^2]
  \le
  \frac{M_{2,W,R}}{T_2}.
\]
Thus Markov's inequality and Young's inequality imply, with probability at least $1-\delta_0$,
\[
  \|r_{\hat a_W}-r_{\mathrm{target}}\|_{W,T_2}^2
  +
  \frac{\lambda}{2}\|\hat a_W\|_2^2
  \le
  \mathcal E_{W,T_2}(a^\sharp)
  +
  C\frac{M_{2,W,R}}{T_2\lambda\delta_0}
  +
  \frac{3\lambda}{2}\|a^\sharp\|_2^2.
\]

Second, compare the weighted empirical seminorm with the weighted population seminorm. Define
\[
  \Sigma_T
  :=
  \frac{1}{T_2}\sum_{i=1}^{T_2}
  W_R(x_i,\zeta_i)\psi_N(x_i)\psi_N(x_i)^\top,
  \qquad
  \Sigma
  :=
  \mathbb E_{x\sim\pi_{\mathrm{ref}},\,\zeta\sim\mathrm{Unif}[-\tau,\tau]}
  \left[
    W_R(x,\zeta)\psi_N(x)\psi_N(x)^\top
  \right],
\]
\[
  u_T
  :=
  \frac{1}{T_2}\sum_{i=1}^{T_2}
  W_R(x_i,\zeta_i)r_{\mathrm{target}}(x_i)\psi_N(x_i),
  \qquad
  u
  :=
  \mathbb E_{x\sim\pi_{\mathrm{ref}},\,\zeta\sim\mathrm{Unif}[-\tau,\tau]}
  \left[
    W_R(x,\zeta)r_{\mathrm{target}}(x)\psi_N(x)
  \right],
\]
and
\[
  s_T
  :=
  \frac{1}{T_2}\sum_{i=1}^{T_2}W_R(x_i,\zeta_i)r_{\mathrm{target}}(x_i)^2,
  \qquad
  s
  :=
  \mathbb E_{x\sim\pi_{\mathrm{ref}},\,\zeta\sim\mathrm{Unif}[-\tau,\tau]}
  \left[
    W_R(x,\zeta)r_{\mathrm{target}}(x)^2
  \right].
\]
Then
\[
  \|r_{\hat a_W}-r_{\mathrm{target}}\|_{W,T_2}^2
  =
  \|r_{\hat a_W}-r_{\mathrm{target}}\|_{W,\mathrm{pop}}^2
  +
  \hat a_W^\top(\Sigma_T-\Sigma)\hat a_W
  -
  2\hat a_W^\top(u_T-u)
  +
  (s_T-s).
\]
The three deviations are controlled by the moments defined above:
\[
  \mathbb E_{\mathrm{sample}}[\|\Sigma_T-\Sigma\|_F^2]
  \le
  \frac{M_{4,W,R}}{T_2},
\]
\[
  \mathbb E_{\mathrm{sample}}[\|u_T-u\|_2^2]
  \le
  \frac{M_{4,W,R}+G_{4,W,R}}{2T_2},
\]
where $2|r_{\mathrm{target}}(x)|^2\|\psi_N(x)\|_2^2\le |r_{\mathrm{target}}(x)|^4+\|\psi_N(x)\|_2^4$, and
\[
  \mathbb E_{\mathrm{sample}}[(s_T-s)^2]
  \le
  \frac{G_{4,W,R}}{T_2}.
\]
Applying Markov's inequality to these three bounds and using Young's inequality for the cross term gives, on an event of probability at least $1-3\delta_0$,
\[
\begin{aligned}
  \|r_{\hat a_W}-r_{\mathrm{target}}\|_{W,\mathrm{pop}}^2
  \le\;&
  \|r_{\hat a_W}-r_{\mathrm{target}}\|_{W,T_2}^2
  +
  \sqrt{\frac{M_{4,W,R}}{T_2\delta_0}}\|\hat a_W\|_2^2 \\
  &+
  \frac{\lambda}{4}\|\hat a_W\|_2^2
  +
  C\frac{M_{4,W,R}+G_{4,W,R}}{T_2\lambda\delta_0}
  +
  \sqrt{\frac{G_{4,W,R}}{T_2\delta_0}}.
\end{aligned}
\]
Combining this with the ridge-optimality bound and taking a union bound yields, with probability at least $1-4\delta_0$,
\[
\begin{aligned}
  &\|r_{\hat a_W}-r_{\mathrm{target}}\|_{W,\mathrm{pop}}^2
  +
  \left(
    \frac{\lambda}{4}
    -
    \sqrt{\frac{M_{4,W,R}}{T_2\delta_0}}
  \right)
  \|\hat a_W\|_2^2 \\
  &\qquad\lesssim
  \mathcal E_{W,T_2}(a^\sharp)
  +
  \frac{M_{2,W,R}+M_{4,W,R}+G_{4,W,R}}{T_2\lambda\delta_0}
  +
  \sqrt{\frac{G_{4,W,R}}{T_2\delta_0}}
  +
  \lambda\|a^\sharp\|_2^2.
\end{aligned}
\]
The lower bound on $\lambda$ makes the coefficient of $\|\hat a_W\|_2^2$ nonnegative, so this term can be dropped. Finally,
\[
  \mathbb E_{x\sim\pi_{\mathrm{ref}},\,\zeta\sim\mathrm{Unif}[-\tau,\tau]}
  \left[
    W_R(x,\zeta)(r_{\hat a_W}(x)-r_{\mathrm{target}}(x))^2
  \right]
  =
  Z_W\|r_{\hat a_W}-r_{\mathrm{target}}\|_{L^2(\nu_W)}^2
\]
converts the weighted population seminorm into an $L^2(\nu_W)$ error.
\end{proof}

\begin{lemma}[Label-weighted truncated ridge estimate]
\label{lem:label-weighted-truncated-ridge}
Under the first-layer recovery condition of Lemma~\ref{lem:second-layer-comparators-from-recovered-features}, if
\[
  \lambda
  \ge
  C\sqrt{\frac{M_{4,\mathrm{wt},R}}{T_2\delta_0}},
\]
then the label-weighted ridge estimator $\hat r_R$ satisfies, with probability at least $1-4\delta_0$ over the ridge-fitting sample,
\[
\begin{aligned}
  \inf_{c\in\mathbb R}
  \|\hat r_R-r^*-c\|_{L^2(\pi_{{\beta_{2}},R})}^2
  \lesssim
  &\frac{1}{Z_{{\beta_{2}},R}Z_\zeta}
  \left(
    e^{(M_{*,R}+\tau)/{\beta_{2}}}(N^{-2}+\varepsilon^2) \right.\\
    &+
    \left.\frac{
      M_{2,\mathrm{wt},R}
      +
      M_{4,\mathrm{wt},R}
      +
      G_{4,\mathrm{wt},R}
    }{
      T_2\lambda\delta_0
    }
    +
    \sqrt{\frac{G_{4,\mathrm{wt},R}}{T_2\delta_0}}
    +
    \lambda N
  \right).
\end{aligned}
\]
\end{lemma}

\begin{proof}
Apply Lemma~\ref{lem:abstract-truncated-weighted-ridge} with
\[
  W_R(x,\zeta)
  =
  \mathbf 1_{B_R}(x)e^{(r^*(x)+\zeta)/{\beta_{2}}},
  \qquad
  r_{\mathrm{target}}(x)=r_{\mathrm{lbl},{\beta_{2}}}^*(x),
  \qquad
  \xi_R(x,\zeta)=\zeta-m_{\zeta,{\beta_{2}}}.
\]
Then $Z_W=Z_{{\beta_{2}},R}Z_\zeta$ and $\nu_W=\pi_{{\beta_{2}},R}$. The required weighted residual orthogonality follows from
\[
  \mathbb E_{\zeta\sim\mathrm{Unif}[-\tau,\tau]}
  \left[
    e^{\zeta/{\beta_{2}}}(\zeta-m_{\zeta,{\beta_{2}}})
  \right]
  =
  0.
\]
Moreover, on $B_R\times[-\tau,\tau]$,
\[
  W_R(x,\zeta)
  \le
  e^{(M_{*,R}+\tau)/{\beta_{2}}},
\]
so the comparator bound in Lemma~\ref{lem:second-layer-comparators-from-recovered-features} gives
\[
  \mathcal E_{W,T_2}(a^\sharp)
  \lesssim
  e^{(M_{*,R}+\tau)/{\beta_{2}}}(N^{-2}+\varepsilon^2).
\]
The same lemma gives $\|a^\sharp\|_2^2\lesssim N$, which turns the comparator regularization term in Lemma~\ref{lem:abstract-truncated-weighted-ridge} into $\lambda N$.
Since $r_{\mathrm{lbl},{\beta_{2}}}^*=r^*+m_{\zeta,{\beta_{2}}}$, taking the infimum over constants converts the $L^2(\pi_{{\beta_{2}},R})$ error for $r_{\mathrm{lbl},{\beta_{2}}}^*$ into the claimed shift-invariant error for $r^*$.
\end{proof}

\begin{lemma}[Surrogate-weighted truncated ridge estimate]
\label{lem:surrogate-weighted-truncated-ridge}
Under the first-layer recovery condition of Lemma~\ref{lem:second-layer-comparators-from-recovered-features}, if
\[
  \lambda
  \ge
  C\sqrt{\frac{M_{4,\mathrm{wt},R}^{(0)}}{T_2\delta_0}},
\]
then, with probability at least $1-4\delta_0$ over the ridge-fitting sample,
\[
  \inf_{c\in\mathbb R}
  \|\hat r_R^{(0)}-r^*-c\|_{L^2(\pi_{{\beta_{2}},R}^{(0)})}^2
  \le
  \mathcal E_R^{(0)},
\]
where
\[
  \mathcal E_R^{(0)}
  :=
  \frac{1}{Z_{{\beta_{2}},R}^{(0)}}
  \left(
    e^{M_{0,R}/{\beta_{2}}}(N^{-2}+\varepsilon^2)
    +
    \frac{
      M_{2,\mathrm{wt},R}^{(0)}
      +
      M_{4,\mathrm{wt},R}^{(0)}
      +
      G_{4,\mathrm{wt},R}^{(0)}
    }{
      T_2\lambda\delta_0
    }
    +
    \sqrt{\frac{G_{4,\mathrm{wt},R}^{(0)}}{T_2\delta_0}}
    +
    \lambda N
  \right).
\]
\end{lemma}

\begin{proof}
Apply Lemma~\ref{lem:abstract-truncated-weighted-ridge} with
\[
  W_R(x,\zeta)
  =
  \mathbf 1_{B_R}(x)e^{r_{a_0}(x)/{\beta_{2}}},
  \qquad
  r_{\mathrm{target}}(x)=r^*(x),
  \qquad
  \xi_R(x,\zeta)=\zeta.
\]
Then $Z_W=Z_{{\beta_{2}},R}^{(0)}$ and $\nu_W=\pi_{{\beta_{2}},R}^{(0)}$. The required weighted residual orthogonality follows from
\[
  \mathbb E_{\zeta\sim\mathrm{Unif}[-\tau,\tau]}[\zeta]=0,
\]
because $W_R$ is independent of $\zeta$. On $B_R$,
\[
  W_R(x,\zeta)
  \le
  e^{M_{0,R}/{\beta_{2}}},
\]
so the comparator bound gives
\[
  \mathcal E_{W,T_2}(a^\sharp)
  \lesssim
  e^{M_{0,R}/{\beta_{2}}}(N^{-2}+\varepsilon^2).
\]
The same lemma gives $\|a^\sharp\|_2^2\lesssim N$, which turns the comparator regularization term in Lemma~\ref{lem:abstract-truncated-weighted-ridge} into $\lambda N$.
The claim follows from Lemma~\ref{lem:abstract-truncated-weighted-ridge}.
\end{proof}

\begin{proof}[Derivation of the surrogate-weighted learning bound]
Let
\[
  \mathcal C_{0,R}({\beta_{2}})
  =
  \frac{e^{M_{0,R}/{\beta_{2}}}}{Z_{{\beta_{2}},R}^{(0)}}.
\]

On the projected truncation set $B_R$, the functions $\|\psi_N(x)\|_2^2$,
$\|\psi_N(x)\|_2^4$, and $|r^*(x)|^4$ are bounded:
\(\psi_N\) depends only on the coordinates in \(S\), and \(B_R\) bounds
\(\|P_Sx\|_2\). Since
$r_{a_0}(x)\le M_{0,R}$ on $B_R$,
\[
  \frac{M_{2,\mathrm{wt},R}^{(0)}}{Z_{{\beta_{2}},R}^{(0)}}
  \lesssim
  e^{M_{0,R}/{\beta_{2}}},
  \qquad
  \frac{M_{4,\mathrm{wt},R}^{(0)}+G_{4,\mathrm{wt},R}^{(0)}}{Z_{{\beta_{2}},R}^{(0)}}
  \lesssim
  e^{M_{0,R}/{\beta_{2}}}.
\]
Also,
\[
  M_{4,\mathrm{wt},R}^{(0)}
  \lesssim
  e^{M_{0,R}/{\beta_{2}}}Z_{{\beta_{2}},R}^{(0)}
  =
  \frac{e^{2M_{0,R}/{\beta_{2}}}}{\mathcal C_{0,R}({\beta_{2}})}.
\]
We choose the regularization parameter at the scale
\begin{equation}
\label{eq:surrogate-weighted-lambda-choice}
  \lambda
  =
  C_{\lambda,\mathrm{surr},R,N}
    e^{M_{0,R}/{\beta_{2}}}
    \left(\mathcal C_{0,R}({\beta_{2}})T_2\delta_0\right)^{-1/2}.
\end{equation}
where \(C_{\lambda,\mathrm{surr},R,N}\) is sufficiently large and independent of
\({\beta_{2}},T_2,\delta_0\). The preceding bound on
\(M_{4,\mathrm{wt},R}^{(0)}\) shows that this choice satisfies the lower bound
required by Lemma~\ref{lem:surrogate-weighted-truncated-ridge}.

Start from Lemma~\ref{lem:surrogate-weighted-truncated-ridge}. The approximation
term is
\[
  \frac{e^{M_{0,R}/{\beta_{2}}}}{Z_{{\beta_{2}},R}^{(0)}}
  (N^{-2}+\varepsilon^2)
  =
  \mathcal C_{0,R}({\beta_{2}})(N^{-2}+\varepsilon^2).
\]
The weighted moment term satisfies
\[
  \frac{
    M_{2,\mathrm{wt},R}^{(0)}
    +
    M_{4,\mathrm{wt},R}^{(0)}
    +
    G_{4,\mathrm{wt},R}^{(0)}
  }{
    Z_{{\beta_{2}},R}^{(0)}T_2\lambda\delta_0
  }
  \lesssim
  \frac{\mathcal C_{0,R}({\beta_{2}})^{1/2}}{\sqrt{T_2\delta_0}}.
\]
For the square-root term, the bound on $G_{4,\mathrm{wt},R}^{(0)}$ gives
\[
  \frac{1}{Z_{{\beta_{2}},R}^{(0)}}
  \sqrt{\frac{G_{4,\mathrm{wt},R}^{(0)}}{T_2\delta_0}}
  \lesssim
  \frac{\mathcal C_{0,R}({\beta_{2}})^{1/2}}{\sqrt{T_2\delta_0}}.
\]
The regularization term is controlled by
\[
  \frac{\lambda N}{Z_{{\beta_{2}},R}^{(0)}}
  \lesssim
  \frac{\mathcal C_{0,R}({\beta_{2}})^{1/2}}{\sqrt{T_2\delta_0}},
\]
where the implicit constant may depend on $N$. Combining these estimates gives
the displayed reward prediction error bound.

The learning-term bound follows from Lemma~\ref{lem:learning-bridge}:
\[
  |T_{\mathrm{learn}}^{(0)}(R)|
  \le
  \frac{2M_R\mathcal D_{\mathrm{surr},R}}{{\beta_{2}}}
  \inf_{c\in\mathbb R}
  \|\hat r_R^{(0)}-r^*-c\|_{L^2(\pi_{{\beta_{2}},R}^{(0)})}.
\]
Taking square roots and using $\sqrt{a+b}\le\sqrt a+\sqrt b$ gives the
surrogate-weighted learning contribution in Theorem~\ref{thm:final-surrogate-weighted-regret}. 
Corollary~\ref{cor:surrogate-weighted-projected-truncation} follows by combining
this bound with Lemma~\ref{lem:surrogate-near-maximizer-mass} below.
\end{proof}

\subsection{Polynomial control of $\mathcal{C}_{0, R}(\beta_2)$}
\label{app:surrogate-partition}

Let \(d_S:=\dim S\), and let \(\gamma_S\) denote the standard Gaussian measure
on \(S\). Write
\[
  K_R := \{z\in S:\|z\|_2\le R\}.
\]
Since the frozen surrogate \(r_{a_0}\) depends only on \(P_Sx\), define
\(f_0:K_R\to\mathbb R\) by
\[
  f_0(z) := r_{a_0}(x)
  \qquad \text{for any } x \text{ such that } P_Sx=z .
\]
Then
\[
  Z_{{\beta_{2}},R}^{(0)}
  =
  \int_{K_R} e^{f_0(z)/{\beta_{2}}}\,\gamma_S(dz),
  \qquad
  M_{0,R}=\sup_{z\in K_R} f_0(z),
\]
and
\[
  \mathcal C_{0,R}({\beta_{2}})
  =
  \left(
    \int_{K_R}
    e^{(f_0(z)-M_{0,R})/{\beta_{2}}}\,\gamma_S(dz)
  \right)^{-1}.
\]

\begin{lemma}[Surrogate near-maximizer mass on the projected ball]
\label{lem:surrogate-near-maximizer-mass}
There exists a finite constant \(\bar C_{0,R}\), independent of
\({\beta_{2}}\), such that for all sufficiently small \({\beta_{2}}>0\),
\[
  \mathcal C_{0,R}({\beta_{2}})
  \le
  \bar C_{0,R}{\beta_{2}}^{-d_S}.
\]
\end{lemma}

\begin{proof}
Since \(r_{a_0}\) is a finite linear combination of polynomial ridge features,
\(f_0\) is a polynomial function on \(S\). Hence \(f_0\) is Lipschitz on the
compact set \(K_R\). Let \(L_{0,R}\) be a Lipschitz constant and set
\(L_+:=\max\{L_{0,R},1\}\).

Let \(z_\star\in K_R\) be a maximizer of \(f_0\). For
\[
  A_{{\beta_{2}}}
  :=
  K_R\cap B\left(z_\star,\frac{{\beta_{2}}}{L_+}\right),
\]
the Lipschitz property gives
\[
  f_0(z)\ge M_{0,R}-{\beta_{2}}
  \qquad \text{for all } z\in A_{{\beta_{2}}}.
\]
Therefore
\[
  Z_{{\beta_{2}},R}^{(0)}
  \ge
  e^{M_{0,R}/{\beta_{2}}-1}\gamma_S(A_{{\beta_{2}}}).
\]

The Gaussian density on \(S\) is bounded below by a positive constant on
\(K_R\). Moreover, since \(K_R\) is a Euclidean ball in \(S\), there exists
\(c_{R,S}>0\) such that
\[
  \operatorname{Vol}_{d_S}
  \left(
    K_R\cap B(z_\star,\rho)
  \right)
  \ge
  c_{R,S}\rho^{d_S}
\]
for all sufficiently small \(\rho>0\), uniformly over \(z_\star\in K_R\).
Applying this with \(\rho={\beta_{2}}/L_+\) yields
\[
  \gamma_S(A_{{\beta_{2}}})
  \ge
  c'_{0,R}{\beta_{2}}^{d_S}
\]
for all sufficiently small \({\beta_{2}}>0\). Hence
\[
  Z_{{\beta_{2}},R}^{(0)}
  \ge
  c''_{0,R}e^{M_{0,R}/{\beta_{2}}}{\beta_{2}}^{d_S}.
\]
Rearranging proves
\[
  \mathcal C_{0,R}({\beta_{2}})
  \le
  \bar C_{0,R}{\beta_{2}}^{-d_S}.
\]
\end{proof}

\begin{remark}[Sharper exponents from local Taylor expansions] \label{rem:sharper_exponents}
The exponent \(d_S\) in Lemma~\ref{lem:surrogate-near-maximizer-mass} is a
worst-case exponent obtained from Lipschitz continuity alone. If the local
Taylor expansion of \(f_0\) around its maximizers is analyzed more precisely,
the exponent can often be improved. For example, a nondegenerate interior
maximizer gives
\[
  \mathcal C_{0,R}({\beta_{2}})
  \lesssim
  {\beta_{2}}^{-d_S/2}.
\]
More generally, the exponent is determined by the leading nonzero negative
terms in the local Taylor expansion near the maximizers, analogously to the
low-temperature analysis for the true reward in
Appendix~\ref{app:one-dimensional-tilted-limit}. These sharper estimates are
not needed for the main value-gap bound.
\end{remark}

\section{Low-temperature simplification for label weighting}
\label{app:label-weighted-low-temp}

This appendix derives the low-temperature moment ratios used to simplify the
label-weighted reward prediction error.
Write
\[
  \alpha:=\frac{1}{2p_{\max}},
  \qquad
  L_*:=B_*+\tau.
\]
Here $s>0$ is a generic temperature parameter; below we apply these estimates
with $s={\beta_{2}}$ and $s={\beta_{2}}/2$. Define
\[
  \pi_{s,R}(dx)
  :=
  \frac{\mathbf 1_{B_R}(x)e^{r^*(x)/s}}{Z_{s,R}}
  \pi_{\mathrm{ref}}(dx),
  \qquad
  Z_{s,R}
  :=
  \mathbb E_{x\sim\pi_{\mathrm{ref}}}
  \left[
    \mathbf 1_{B_R}(x)e^{r^*(x)/s}
  \right].
\]
Also define the exponentially weighted uniform-noise normalization at temperature $s$ by
\[
  Z_\zeta(s)
  :=
  \mathbb E_{\zeta\sim\mathrm{Unif}[-\tau,\tau]}
  \left[
    e^{\zeta/s}
  \right],
\]
and the corresponding probability measure by
\[
  \nu_{\zeta,s}(d\zeta)
  :=
  \frac{e^{\zeta/s}}{Z_\zeta(s)}
  \mathrm{Unif}[-\tau,\tau](d\zeta),
  \qquad
  m_{\zeta,s}
  :=
  \mathbb E_{\zeta\sim\nu_{\zeta,s}}[\zeta].
\]
Thus, specializing the generic temperature $s$ to the deployment temperature ${\beta_{2}}$, we write $Z_\zeta=Z_\zeta({\beta_{2}})$ and $m_{\zeta,{\beta_{2}}}=m_{\zeta,s}|_{s={\beta_{2}}}$.

\begin{lemma}[Exact formulas for the exponentially weighted uniform noise]
\label{lem:noise-exact-formulas}
For every $s>0$,
\[
  Z_\zeta(s)
  =
  \frac{s}{\tau}\sinh\left(\frac{\tau}{s}\right),
  \qquad
  m_{\zeta,s}
  =
  \tau\operatorname{coth}\left(\frac{\tau}{s}\right)-s.
\]
Moreover,
\[
  \operatorname{Var}_{\zeta\sim\nu_{\zeta,s}}[\zeta]
  =
  s^2-\tau^2\operatorname{csch}^2\left(\frac{\tau}{s}\right).
\]
At the two temperatures ${\beta_{2}}$ and ${\beta_{2}}/2$,
\[
  \frac{Z_\zeta({\beta_{2}}/2)}{Z_\zeta({\beta_{2}})}
  =
  \cosh\left(\frac{\tau}{{\beta_{2}}}\right),
\]
and
\[
  V_{\zeta,{\beta_{2}}}^{(2)}
  :=
  \mathbb E_{\zeta\sim\nu_{\zeta,{\beta_{2}}/2}}
  \left[
    (\zeta-m_{\zeta,{\beta_{2}}})^2
  \right]
  =
  \frac{{\beta_{2}}^2}{2}
  -
  {\beta_{2}}\tau\operatorname{csch}\left(\frac{2\tau}{{\beta_{2}}}\right).
\]
Consequently, as ${\beta_{2}}\downarrow0$,
\[
  Z_\zeta({\beta_{2}})
  =
  \frac{{\beta_{2}}}{2\tau}e^{\tau/{\beta_{2}}}
  \left(1+o(1)\right),
  \qquad
  m_{\zeta,{\beta_{2}}}
  =
  \tau-{\beta_{2}}+O(e^{-2\tau/{\beta_{2}}}),
\]
\[
  \frac{Z_\zeta({\beta_{2}}/2)}{Z_\zeta({\beta_{2}})}
  =
  \frac12 e^{\tau/{\beta_{2}}}\left(1+o(1)\right),
  \qquad
  V_{\zeta,{\beta_{2}}}^{(2)}
  =
  \frac{{\beta_{2}}^2}{2}+O\left({\beta_{2}}\tau e^{-2\tau/{\beta_{2}}}\right).
\]
\end{lemma}

\begin{proof}
The identity for $Z_\zeta(s)$ follows directly from
\[
  \mathbb E_{\zeta\sim\mathrm{Unif}[-\tau,\tau]}
  \left[
    e^{\zeta/s}
  \right]
  =
  \frac{1}{2\tau}\int_{-\tau}^{\tau}e^{z/s}\,dz.
\]
Let
\[
  \Lambda_\zeta(\eta)
  :=
  \log
  \mathbb E_{\zeta\sim\mathrm{Unif}[-\tau,\tau]}
  \left[
    e^{\eta\zeta}
  \right].
\]
Then $\nu_{\zeta,s}$ is the exponentially weighted uniform law with
$\eta=1/s$. Differentiating this log-normalization gives
\[
  \Lambda_\zeta'(\eta)
  =
  \mathbb E_{\zeta\sim\nu_{\zeta,1/\eta}}[\zeta],
  \qquad
  \Lambda_\zeta''(\eta)
  =
  \operatorname{Var}_{\zeta\sim\nu_{\zeta,1/\eta}}[\zeta].
\]
Differentiating
$\Lambda_\zeta(\eta)=\log(\sinh(\tau\eta)/(\tau\eta))$ and setting
$\eta=1/s$ gives the formulas for $m_{\zeta,s}$ and
$\operatorname{Var}_{\zeta\sim\nu_{\zeta,s}}[\zeta]$. The ratio
$Z_\zeta({\beta_{2}}/2)/Z_\zeta({\beta_{2}})=\cosh(\tau/{\beta_{2}})$ follows from
$\sinh(2a)=2\sinh(a)\cosh(a)$. Finally,
\[
  V_{\zeta,{\beta_{2}}}^{(2)}
  =
  \operatorname{Var}_{\zeta\sim\nu_{\zeta,{\beta_{2}}/2}}[\zeta]
  +
  \left(m_{\zeta,{\beta_{2}}/2}-m_{\zeta,{\beta_{2}}}\right)^2,
\]
where
\[
  m_{\zeta,{\beta_{2}}/2}-m_{\zeta,{\beta_{2}}}
  =
  \frac{{\beta_{2}}}{2}
  -
  \tau\operatorname{csch}\left(\frac{2\tau}{{\beta_{2}}}\right).
\]
Substituting this identity and
$\operatorname{Var}_{\zeta\sim\nu_{\zeta,{\beta_{2}}/2}}[\zeta]
={\beta_{2}}^2/4-\tau^2\operatorname{csch}^2(2\tau/{\beta_{2}})$ gives the displayed
formula for $V_{\zeta,{\beta_{2}}}^{(2)}$. The asymptotics follow from
$\sinh a\sim e^a/2$, $\cosh a\sim e^a/2$, $\operatorname{coth} a=1+O(e^{-2a})$, and
$\operatorname{csch} a=O(e^{-a})$ as $a\to\infty$.
\end{proof}

The next lemma gives the truncated Laplace estimate needed for feature
moments.
Let
\[
  S_0:=S\cap(\theta^*)^\perp .
\]
Let \(\gamma_{S_0}\) denote the standard Gaussian measure on \(S_0\). For
bounded continuous \(F:B_R\to\mathbb R\) that depend only
on \(P_Sx\), define
\[
\begin{aligned}
  \mathcal A_R[F]
  :=
  \sum_{i\in I_{\max}}
  &\phi(u_i)
  \left(
    \int_{\mathbb R}e^{-c_i t^{2p_{\max}}}\,dt
  \right)
  \\
  &\times
  \int_{S_0}
  \mathbf 1_{\{u_i^2+\|z\|_2^2\le R^2\}}
  F(u_i\theta^*+z)\,d\gamma_{S_0}(z).
\end{aligned}
\]
Set
\[
  A_R:=\mathcal A_R[1].
\]
Since \(R\) is taken sufficiently large, \(R>\max_{i\in I_{\max}}|u_i|\), so
\(A_R>0\). Define
\[
  \mu_{0,R}[F]:=\frac{\mathcal A_R[F]}{A_R}.
\]

\begin{lemma}[Truncated feature Laplace asymptotics]
\label{lem:truncated-feature-laplace}
For every bounded continuous function $F:B_R\to\mathbb R$ that depends only on the projected coordinate $P_Sx$,
\[
  \mathbb E_{x\sim\pi_{\mathrm{ref}}}
  \left[
    \mathbf 1_{B_R}(x)e^{r^*(x)/s}F(x)
  \right]
  =
  e^{B_*/s}s^\alpha
  \left(
    \mathcal A_R[F]+o(1)
  \right)
  \qquad(s\downarrow0).
\]
In particular,
\[
  Z_{s,R}
  =
  e^{B_*/s}s^\alpha(A_R+o(1)),
  \qquad
  \mathbb E_{x\sim\pi_{s,R}}[F(x)]
  \to
  \mu_{0,R}[F].
\]
Moreover, if $F_s:B_R\to\mathbb R$ are bounded continuous functions that
depend only on \(P_Sx\) and $\|F_s-F_0\|_{L^\infty(B_R)}\to0$, then
\[
  \mathbb E_{x\sim\pi_{s,R}}[F_s(x)]
  \to
  \mu_{0,R}[F_0].
\]
\end{lemma}

\begin{proof}
Let
\[
  S_0:=S\cap(\theta^*)^\perp .
\]
Write the projected coordinates as
\[
  P_Sx=u\theta^*+z,
  \qquad
  u:=\langle\theta^*,x\rangle,
  \qquad
  z\in S_0 .
\]
Under \(x\sim\pi_{\mathrm{ref}}\), the coordinates \(u\sim\mathcal N(0,1)\)
and \(z\sim\gamma_{S_0}\) are independent, where \(\gamma_{S_0}\) denotes
the standard Gaussian measure on \(S_0\). Since
\[
  B_R=\{x:\|P_Sx\|_2\le R\},
\]
the truncation event is
\[
  u^2+\|z\|_2^2\le R^2 .
\]

Define
\[
  I_s[F]
  :=
  \mathbb E_{x\sim\pi_{\mathrm{ref}}}
  \left[
    \mathbf 1_{B_R}(x)e^{r^*(x)/s}F(x)
  \right].
\]
Using \(r^*(x)=\sigma^*(\langle\theta^*,x\rangle)=\sigma^*(u)\), we get
\[
  I_s[F]
  =
  \int_{\mathbb R}\int_{S_0}
  \mathbf 1_{\{u^2+\|z\|_2^2\le R^2\}}
  e^{\sigma^*(u)/s}
  F(u\theta^*+z)
  \phi(u)\,d\gamma_{S_0}(z)\,du .
\]
The indicator implies \(|u|\le R\). Choose disjoint neighborhoods \(U_j\)
of the global maximizers of \(\sigma^*\) that intersect \([-R,R]\). For
each \(i\in I_{\max}\), choose \(U_i\) small enough that
\(U_i\subset(-R,R)\), which is possible because \(R>|u_i|\).

Let \(I_s^{\mathrm{off}}[F]\) be the part of \(I_s[F]\) with
\(u\notin\cup_j U_j\). By continuity of \(\sigma^*\) on the compact set
\([-R,R]\setminus\cup_j U_j\), there is \(\eta>0\) such that
\(\sigma^*(u)\le B_*-\eta\) on this set. Hence
\[
  |I_s^{\mathrm{off}}[F]|
  \le
  \|F\|_{L^\infty(B_R)}e^{(B_*-\eta)/s},
\]
and therefore
\[
  e^{-B_*/s}s^{-\alpha}|I_s^{\mathrm{off}}[F]|
  \to
  0 .
\]

Now fix \(i\in I_{\max}\), and let \(I_{s,i}^{\mathrm{loc}}[F]\) be the
part of \(I_s[F]\) with \(u\in U_i\). Set \(u=u_i+s^\alpha t\). Then
\[
  \frac{\sigma^*(u_i+s^\alpha t)-B_*}{s}
  \to
  -c_i t^{2p_{\max}}
\]
for every fixed \(t\). After this change of variables,
\[
\begin{aligned}
  e^{-B_*/s}s^{-\alpha}I_{s,i}^{\mathrm{loc}}[F]
  =
  \int_{T_{s,i}}\int_{S_0}
  &\mathbf 1_{\{(u_i+s^\alpha t)^2+\|z\|_2^2\le R^2\}}
  e^{(\sigma^*(u_i+s^\alpha t)-B_*)/s}
  \\
  &\times
  F((u_i+s^\alpha t)\theta^*+z)
  \phi(u_i+s^\alpha t)\,d\gamma_{S_0}(z)\,dt,
\end{aligned}
\]
where \(T_{s,i}:=\{t:u_i+s^\alpha t\in U_i\}\).

After shrinking \(U_i\) if necessary, the Taylor expansion at \(u_i\)
implies
\[
  \sigma^*(u_i+h)
  \le
  B_*-\frac{c_i}{2}h^{2p_{\max}}
\]
for \(u_i+h\in U_i\). Thus the absolute value of the last integrand is
bounded by
\[
  \|F\|_{L^\infty(B_R)}
  e^{-(c_i/2)t^{2p_{\max}}},
\]
which is integrable with respect to \(dt\,d\gamma_{S_0}(z)\).

For \(\gamma_{S_0}\)-almost every \(z\), the point \(z\) is not on the
boundary
\[
  \|z\|_2^2=R^2-u_i^2 .
\]
Hence
\[
  \mathbf 1_{\{(u_i+s^\alpha t)^2+\|z\|_2^2\le R^2\}}
  \to
  \mathbf 1_{\{u_i^2+\|z\|_2^2\le R^2\}} .
\]
The continuity of \(F\) gives
\[
  F((u_i+s^\alpha t)\theta^*+z)
  \to
  F(u_i\theta^*+z).
\]
Dominated convergence gives
\[
  e^{-B_*/s}s^{-\alpha}I_{s,i}^{\mathrm{loc}}[F]
  \to
  \mathcal A_{R,i}[F],
\]
where
\[
\begin{aligned}
  \mathcal A_{R,i}[F]
  :=
  &\phi(u_i)
  \left(
    \int_{\mathbb R}e^{-c_i t^{2p_{\max}}}\,dt
  \right)
  \\
  &\times
  \int_{S_0}
  \mathbf 1_{\{u_i^2+\|z\|_2^2\le R^2\}}
  F(u_i\theta^*+z)\,d\gamma_{S_0}(z).
\end{aligned}
\]

If \(u_j\) is a global maximizer with \(j\notin I_{\max}\), then
\(p_j<p_{\max}\). The same local argument with the scale
\(s^{1/(2p_j)}\) gives
\[
  I_{s,j}^{\mathrm{loc}}[F]
  =
  O\left(e^{B_*/s}s^{1/(2p_j)}\right)
  =
  o\left(e^{B_*/s}s^\alpha\right).
\]
Combining the off-neighborhood estimate, the dominant local limits, and
the non-dominant local bounds yields
\[
  I_s[F]
  =
  e^{B_*/s}s^\alpha
  \left(
    \sum_{i\in I_{\max}}\mathcal A_{R,i}[F]+o(1)
  \right)
  =
  e^{B_*/s}s^\alpha
  \left(
    \mathcal A_R[F]+o(1)
  \right).
\]

Taking \(F\equiv1\) gives the formula for \(Z_{s,R}\), and division gives
\[
  \mathbb E_{x\sim\pi_{s,R}}[F(x)]
  \to
  \mu_{0,R}[F].
\]
The uniformly convergent version follows by writing
\[
  \left|
    \mathbb E_{x\sim\pi_{s,R}}[F_s(x)]
    -
    \mathbb E_{x\sim\pi_{s,R}}[F_0(x)]
  \right|
  \le
  \|F_s-F_0\|_{L^\infty(B_R)}
\]
and applying the fixed-\(F\) result to \(F_0\).
\end{proof}

\begin{lemma}[Low-temperature normalized label-weighted moments]
\label{lem:label-weighted-moment-ratios-low-temp}
Recall the limiting functional $\mu_{0,R}$ from
Lemma~\ref{lem:truncated-feature-laplace}. Define
\[
  \Psi_{2,R,N}:=
  \mu_{0,R}[\|\psi_N\|_2^2],
  \qquad
  \Psi_{4,R,N}:=
  \mu_{0,R}[\|\psi_N\|_2^4].
\]
With $M_{2,\mathrm{wt},R}$, $M_{4,\mathrm{wt},R}$, and
$G_{4,\mathrm{wt},R}$ defined in Appendix~\ref{app:weighted-ridge-learning}, as
${\beta_{2}}\downarrow0$,
\[
  Z_{{\beta_{2}},R}
  =
  e^{B_*/{\beta_{2}}}{\beta_{2}}^\alpha(A_R+o(1)),
\]
\[
  \frac{Z_{{\beta_{2}}/2,R}}{Z_{{\beta_{2}},R}}
  =
  2^{-\alpha}e^{B_*/{\beta_{2}}}(1+o(1)).
\]
Since $M_{*,R}=B_*$,
\[
  \frac{e^{(M_{*,R}+\tau)/{\beta_{2}}}}{Z_{{\beta_{2}},R}Z_\zeta({\beta_{2}})}
  =
  \frac{2\tau}{A_R}{\beta_{2}}^{-(\alpha+1)}(1+o(1)).
\]
Moreover,
\[
  \frac{M_{2,\mathrm{wt},R}}{Z_{{\beta_{2}},R}Z_\zeta({\beta_{2}})}
  =
  2^{-\alpha-2}
  \Psi_{2,R,N}
  {\beta_{2}}^2e^{L_*/{\beta_{2}}}
  (1+o(1)),
\]
\[
  \frac{M_{4,\mathrm{wt},R}}{Z_{{\beta_{2}},R}Z_\zeta({\beta_{2}})}
  =
  2^{-\alpha-1}
  \Psi_{4,R,N}
  e^{L_*/{\beta_{2}}}
  (1+o(1)),
\]
and
\[
  \frac{G_{4,\mathrm{wt},R}}{Z_{{\beta_{2}},R}Z_\zeta({\beta_{2}})}
  =
  O\left(e^{L_*/{\beta_{2}}}\right).
\]
If $B_*+\tau\ne0$, the last display can be sharpened to
\[
  \frac{G_{4,\mathrm{wt},R}}{Z_{{\beta_{2}},R}Z_\zeta({\beta_{2}})}
  =
  2^{-\alpha-1}
  |B_*+\tau|^4
  e^{L_*/{\beta_{2}}}
  (1+o(1)).
\]
\end{lemma}

\begin{proof}
First reduce the weighted moments to normalized expectations. The random variables
$x\sim\pi_{\mathrm{ref}}$ and $\zeta\sim\mathrm{Unif}[-\tau,\tau]$ are
independent, and
\[
  e^{2(r^*(x)+\zeta)/{\beta_{2}}}
  =
  e^{r^*(x)/({\beta_{2}}/2)}e^{\zeta/({\beta_{2}}/2)}.
\]
Thus each unnormalized weighted moment splits into an $x$-integral normalized
by $Z_{{\beta_{2}}/2,R}$ and a noise integral normalized by $Z_\zeta({\beta_{2}}/2)$. After
division by $Z_{{\beta_{2}},R}Z_\zeta({\beta_{2}})$,
\[
  \frac{M_{2,\mathrm{wt},R}}{Z_{{\beta_{2}},R}Z_\zeta({\beta_{2}})}
  =
  \frac{Z_{{\beta_{2}}/2,R}}{Z_{{\beta_{2}},R}}
  \frac{Z_\zeta({\beta_{2}}/2)}{Z_\zeta({\beta_{2}})}
  V_{\zeta,{\beta_{2}}}^{(2)}
  \mathbb E_{x\sim\pi_{{\beta_{2}}/2,R}}
  [\|\psi_N(x)\|_2^2],
\]
\[
  \frac{M_{4,\mathrm{wt},R}}{Z_{{\beta_{2}},R}Z_\zeta({\beta_{2}})}
  =
  \frac{Z_{{\beta_{2}}/2,R}}{Z_{{\beta_{2}},R}}
  \frac{Z_\zeta({\beta_{2}}/2)}{Z_\zeta({\beta_{2}})}
  \mathbb E_{x\sim\pi_{{\beta_{2}}/2,R}}
  [\|\psi_N(x)\|_2^4],
\]
and
\[
  \frac{G_{4,\mathrm{wt},R}}{Z_{{\beta_{2}},R}Z_\zeta({\beta_{2}})}
  =
  \frac{Z_{{\beta_{2}}/2,R}}{Z_{{\beta_{2}},R}}
  \frac{Z_\zeta({\beta_{2}}/2)}{Z_\zeta({\beta_{2}})}
  \mathbb E_{x\sim\pi_{{\beta_{2}}/2,R}}
  [|r^*(x)+m_{\zeta,{\beta_{2}}}|^4].
\]
Here the noise factor in the first display is
\[
  \mathbb E_{\zeta\sim\nu_{\zeta,{\beta_{2}}/2}}
  [(\zeta-m_{\zeta,{\beta_{2}}})^2]
  =
  V_{\zeta,{\beta_{2}}}^{(2)}.
\]
We now take the low-temperature limit. The asymptotic for $Z_{{\beta_{2}},R}$ is
Lemma~\ref{lem:truncated-feature-laplace} with $F\equiv1$. Applying the same
lemma at temperature ${\beta_{2}}/2$ gives
\[
  Z_{{\beta_{2}}/2,R}
  =
  e^{2B_*/{\beta_{2}}}({\beta_{2}}/2)^\alpha(A_R+o(1)),
\]
which implies the displayed ratio $Z_{{\beta_{2}}/2,R}/Z_{{\beta_{2}},R}$. Since \(R\) is taken sufficiently large, \(R>|u_i|\) for all
\(i\in I_{\max}\). Hence the truncated ball contains a global maximizer of
\(r^*\), and \(M_{*,R}=B_*\).
Combining the asymptotic for
$Z_{{\beta_{2}},R}$ with Lemma~\ref{lem:noise-exact-formulas} gives the
normalization ratio involving $e^{(M_{*,R}+\tau)/{\beta_{2}}}$.

For the weighted moments, use the normalized identities proved above.
Lemma~\ref{lem:truncated-feature-laplace} gives
\[
  \mathbb E_{x\sim\pi_{{\beta_{2}}/2,R}}[\|\psi_N(x)\|_2^2]
  \to
  \Psi_{2,R,N},
  \qquad
  \mathbb E_{x\sim\pi_{{\beta_{2}}/2,R}}[\|\psi_N(x)\|_2^4]
  \to
  \Psi_{4,R,N}.
\]
The noise formulas give
\[
  \frac{Z_\zeta({\beta_{2}}/2)}{Z_\zeta({\beta_{2}})}
  =
  \frac12 e^{\tau/{\beta_{2}}}(1+o(1)),
  \qquad
  V_{\zeta,{\beta_{2}}}^{(2)}
  =
  \frac{{\beta_{2}}^2}{2}(1+o(1)).
\]
This proves the $M_{2,\mathrm{wt},R}$ and $M_{4,\mathrm{wt},R}$ formulas.
Finally,
\[
  m_{\zeta,{\beta_{2}}}\to\tau,
\]
so the functions
\[
  F_{\beta_{2}}(x):=|r^*(x)+m_{\zeta,{\beta_{2}}}|^4
\]
converge uniformly on $B_R$ to $F_0(x)=|r^*(x)+\tau|^4$. The uniform version of
Lemma~\ref{lem:truncated-feature-laplace} gives
\[
  \mathbb E_{x\sim\pi_{{\beta_{2}}/2,R}}[|r^*(x)+m_{\zeta,{\beta_{2}}}|^4]
  \to
  \mu_{0,R}[|r^*+\tau|^4].
\]
Under $\mu_{0,R}$, the coordinate $\langle\theta^*,x\rangle$ belongs to
$\{u_i:i\in I_{\max}\}$ almost surely. Since
$r^*(x)=\sigma^*(\langle\theta^*,x\rangle)$, this implies
$r^*(x)=B_*$ $\mu_{0,R}$-almost surely. Therefore the last limit is
$|B_*+\tau|^4$. This gives the sharpened asymptotic when $B_*+\tau\ne0$.
Without the nondegeneracy condition, the same convergence only gives that this
expectation is bounded, which yields the displayed upper bound.
\end{proof}

\begin{proof}[Derivation of the label-weighted learning bound]
Start from Lemma~\ref{lem:label-weighted-truncated-ridge}:
\[
\begin{aligned}
  \inf_{c\in\mathbb R}
  \|\hat r_R-r^*-c\|_{L^2(\pi_{{\beta_{2}},R})}^2
  \lesssim
  \frac{1}{Z_{{\beta_{2}},R}Z_\zeta({\beta_{2}})}
  \Bigg(
  &e^{(M_{*,R}+\tau)/{\beta_{2}}}(N^{-2}+\varepsilon^2)
  \\
  &+
  \frac{M_{2,\mathrm{wt},R}+M_{4,\mathrm{wt},R}+G_{4,\mathrm{wt},R}}
  {T_2\lambda\delta_0}
  \\
  &+
  \sqrt{\frac{G_{4,\mathrm{wt},R}}{T_2\delta_0}}
  +
  \lambda N
  \Bigg).
\end{aligned}
\]
The approximation term is controlled by
\[
  \frac{e^{(M_{*,R}+\tau)/{\beta_{2}}}}{Z_{{\beta_{2}},R}Z_\zeta({\beta_{2}})}
  \lesssim
  {\beta_{2}}^{-(\alpha+1)}.
\]
The moment terms are controlled by
Lemma~\ref{lem:label-weighted-moment-ratios-low-temp}:
\[
  \frac{M_{2,\mathrm{wt},R}}{Z_{{\beta_{2}},R}Z_\zeta({\beta_{2}})}
  \lesssim
  {\beta_{2}}^2e^{L_*/{\beta_{2}}},
  \qquad
  \frac{M_{4,\mathrm{wt},R}+G_{4,\mathrm{wt},R}}{Z_{{\beta_{2}},R}Z_\zeta({\beta_{2}})}
  \lesssim
  e^{L_*/{\beta_{2}}}.
\]
For the square-root term, the same lemma implies
\[
  \frac{G_{4,\mathrm{wt},R}}{Z_{{\beta_{2}},R}Z_\zeta({\beta_{2}})}
  \lesssim
  e^{L_*/{\beta_{2}}},
  \qquad
  Z_{{\beta_{2}},R}Z_\zeta({\beta_{2}})
  \lesssim
  e^{L_*/{\beta_{2}}}{\beta_{2}}^{\alpha+1}.
\]
Multiplying these two bounds gives
\[
  G_{4,\mathrm{wt},R}
  \lesssim
  e^{2L_*/{\beta_{2}}}{\beta_{2}}^{\alpha+1}.
\]
The lower bound
\[
  Z_{{\beta_{2}},R}Z_\zeta({\beta_{2}})
  \gtrsim
  e^{L_*/{\beta_{2}}}{\beta_{2}}^{\alpha+1}
\]
then gives
\[
  \frac{1}{Z_{{\beta_{2}},R}Z_\zeta({\beta_{2}})}
  \sqrt{\frac{G_{4,\mathrm{wt},R}}{T_2\delta_0}}
  \lesssim
  \frac{{\beta_{2}}^{-(\alpha+1)/2}}{\sqrt{T_2\delta_0}}.
\]
Finally,
\[
  \frac{\lambda N}{Z_{{\beta_{2}},R}Z_\zeta({\beta_{2}})}
  \lesssim
  \lambda N{\beta_{2}}^{-(\alpha+1)}e^{-L_*/{\beta_{2}}}.
\]
Combining these estimates gives, for every $\lambda$ satisfying the lower bound
in Lemma~\ref{lem:label-weighted-truncated-ridge},
\[
\begin{aligned}
  \inf_{c\in\mathbb R}
  \|\hat r_R-r^*-c\|_{L^2(\pi_{{\beta_{2}},R})}^2
  \lesssim\;&
  C_{\mathrm{app},R}
  {\beta_{2}}^{-(\alpha+1)}(N^{-2}+\varepsilon^2)
  \\
  &+
  C_{\mathrm{mom},R,N}
  \frac{e^{L_*/{\beta_{2}}}}{T_2\lambda\delta_0}
  (1+{\beta_{2}}^2)
  \\
  &+
  C_{\sqrt G,R}
  \frac{{\beta_{2}}^{-(\alpha+1)/2}}{\sqrt{T_2\delta_0}}
  +
  C_{\lambda,R}
  \lambda N{\beta_{2}}^{-(\alpha+1)}e^{-L_*/{\beta_{2}}}.
\end{aligned}
\]
The lower bound on $\lambda$ required by
Lemma~\ref{lem:label-weighted-truncated-ridge} is
$\lambda\ge C\sqrt{M_{4,\mathrm{wt},R}/(T_2\delta_0)}$. By
Lemma~\ref{lem:label-weighted-moment-ratios-low-temp},
\[
  M_{4,\mathrm{wt},R}
  =
  \frac{M_{4,\mathrm{wt},R}}{Z_{{\beta_{2}},R}Z_\zeta({\beta_{2}})}
  Z_{{\beta_{2}},R}Z_\zeta({\beta_{2}})
  \lesssim
  {\beta_{2}}^{\alpha+1}e^{2L_*/{\beta_{2}}},
\]
and therefore it is sufficient that
\[
  \lambda
  \ge
  C_{\lambda,\mathrm{low},R,N}
  \frac{
    {\beta_{2}}^{(\alpha+1)/2}e^{L_*/{\beta_{2}}}
  }{
    \sqrt{T_2\delta_0}
  }.
\]
We choose the regularization parameter at the scale
\begin{equation}
\label{eq:label-weighted-lambda-choice}
  \lambda
  =
  C_{\lambda,\mathrm{choice},R,N}
    {\beta_{2}}^{(\alpha+1)/2}e^{L_*/{\beta_{2}}}
    \left(T_2\delta_0\right)^{-1/2}.
\end{equation}
where \(C_{\lambda,\mathrm{choice},R,N}\) is sufficiently large and independent
of \({\beta_{2}},T_2,\delta_0\). This choice satisfies the lower bound above.
Substituting this value of \(\lambda\) into the intermediate prediction-error
bound gives
\[
  \inf_{c\in\mathbb R}
  \|\hat r_R-r^*-c\|_{L^2(\pi_{{\beta_{2}},R})}^2
  \lesssim
  C_{\mathrm{app},R}
  {\beta_{2}}^{-(\alpha+1)}(N^{-2}+\varepsilon^2)
  +
  C_{\mathrm{stat},R,N}
  \frac{{\beta_{2}}^{-(\alpha+1)/2}}{\sqrt{T_2\delta_0}}.
\]
Finally, Lemma~\ref{lem:learning-bridge} gives
\[
  |T_{\mathrm{learn}}(R)|
  \le
  \frac{2M_R\mathcal D_{\mathrm{lbl},R}}{{\beta_{2}}}
  \left(
    \inf_{c\in\mathbb R}
    \|\hat r_R-r^*-c\|_{L^2(\pi_{{\beta_{2}},R})}^2
  \right)^{1/2}.
\]
Applying $\sqrt{a+b}\le\sqrt a+\sqrt b$ to the prediction-error bound yields
the stated learning bound.
\end{proof}

\section{Admissible deployment temperatures}
\label{app:temperature-selection}

This appendix makes explicit the admissible temperature sets referenced in
Section~\ref{subsec:regret_bounds_projected_truncation}. Write
\[
  \mathcal R_R^{\mathrm{lbl}}:=\mathcal R_R,
  \qquad
  \mathcal R_R^{\mathrm{surr}}:=\mathcal R_R^{(0)},
  \qquad
  \rho_N:=N^{-1}+\varepsilon.
\]
For label weighting, define
\[
  \mathcal L_{\mathrm{lbl}}(\beta)
  :=
  \beta^{-(\alpha+3)/2}\rho_N
  +
  \beta^{-(\alpha+5)/4}(T_2\delta_0)^{-1/4}.
\]
For surrogate weighting, assume that for some \(\alpha_0\ge0\),
\[
  \mathcal C_{0,R}(\beta)\lesssim \beta^{-\alpha_0}
\]
for all sufficiently small \(\beta>0\). Lemma~\ref{lem:surrogate-near-maximizer-mass}
gives the valid worst-case choice \(\alpha_0=d_S\), while
Remark~\ref{rem:sharper_exponents} explains when smaller exponents are
available. Under this polynomial control, define
\[
  \mathcal L_{\mathrm{surr}}(\beta)
  :=
  \beta^{-1-\alpha_0/2}\rho_N
  +
  \beta^{-1-\alpha_0/4}(T_2\delta_0)^{-1/4}.
\]

We assume a deterministic coverage envelope
\(\overline{\mathcal D}_{w,R}:(0,\overline\beta_2]\to[1,\infty]\) such
that
\[
  \mathcal D_{w,R}(\beta)
  \le
  \overline{\mathcal D}_{w,R}(\beta)
\]
for every considered deployment temperature \(\beta\). As in standard
coverage or concentrability assumptions, this envelope is treated as a
problem-dependent quantity rather than estimated from the same
ridge-fitting samples.

For a tolerance \(\eta>0\), define the admissible temperature set
\[
  \mathcal S_w^{\mathrm{adm}}(\eta;N,T_2,\varepsilon,\delta_0)
  :=
  \left\{
    \beta\in(0,\overline\beta_2]:
    M_R\overline{\mathcal D}_{w,R}(\beta)\mathcal L_w(\beta)
    \le \eta
  \right\}.
\]
Here \(w\in\{\mathrm{lbl},\mathrm{surr}\}\), with
\(\mathcal L_w=\mathcal L_{\mathrm{lbl}}\) or
\(\mathcal L_{\mathrm{surr}}\), respectively.

\begin{corollary}[Admissible deployment temperature]
\label{cor:admissible-deployment-temperature}
Fix \(w\in\{\mathrm{lbl},\mathrm{surr}\}\) and \(\eta>0\). On the event of
Theorem~\ref{thm:final-label-weighted-regret} for \(w=\mathrm{lbl}\), or
Corollary~\ref{cor:surrogate-weighted-projected-truncation} for
\(w=\mathrm{surr}\), every
\(\beta_2\in\mathcal S_w^{\mathrm{adm}}(\eta;N,T_2,\varepsilon,\delta_0)\)
satisfies
\[
  |\mathcal R_R^w|
  \lesssim
  \Gamma_{R,\overline\beta_2}(\beta_2,\beta^\ast)+\eta.
\]
Consequently, if
\[
  \beta_{2,\eta}^w
  \in
  \arg\min_{\beta\in
  \mathcal S_w^{\mathrm{adm}}(\eta;N,T_2,\varepsilon,\delta_0)}
  \Gamma_{R,\overline\beta_2}(\beta,\beta^\ast),
\]
then
\[
  |\mathcal R_R^w|
  \lesssim
  \inf_{\beta\in
  \mathcal S_w^{\mathrm{adm}}(\eta;N,T_2,\varepsilon,\delta_0)}
  \Gamma_{R,\overline\beta_2}(\beta,\beta^\ast)
  +
  \eta.
\]
\end{corollary}

\begin{proof}
By definition of the admissible temperature set,
\[
  M_R\overline{\mathcal D}_{w,R}(\beta_2)\mathcal L_w(\beta_2)
  \le
  \eta.
\]
For \(w=\mathrm{lbl}\), this upper bounds the learning term in
Theorem~\ref{thm:final-label-weighted-regret}; for
\(w=\mathrm{surr}\), it upper bounds the learning term in
Corollary~\ref{cor:surrogate-weighted-projected-truncation}. Substituting
this into the corresponding bound gives the first claim. The second claim
follows by minimizing the remaining
\(\Gamma_{R,\overline\beta_2}\) term over the admissible temperature set.
\end{proof}

\begin{remark}[Resource dependence of the admissible temperature set]
\label{rem:admissible-set-resource-dependence}
Fix the coverage envelope \(\overline{\mathcal D}_{w,R}\), the tolerance
\(\eta\), the failure parameter \(\delta_0\), and, in the surrogate case,
the exponent \(\alpha_0\). If
\[
  N'^{-1}+\varepsilon'\le N^{-1}+\varepsilon,
  \qquad
  T_2'\ge T_2,
\]
then
\[
  \mathcal S_w^{\mathrm{adm}}(\eta;N,T_2,\varepsilon,\delta_0)
  \subseteq
  \mathcal S_w^{\mathrm{adm}}(\eta;N',T_2',\varepsilon',\delta_0).
\]
Thus, under a fixed coverage envelope, improving feature recovery or
increasing the number of second-stage samples enlarges the admissible
temperature region.

In particular, for any fixed candidate temperature
\(\beta_0\in(0,\overline\beta_2]\) with
\(\overline{\mathcal D}_{w,R}(\beta_0)<\infty\), the temperature
\(\beta_0\) belongs to
\(\mathcal S_w^{\mathrm{adm}}(\eta;N,T_2,\varepsilon,\delta_0)\) for all
sufficiently large \(N\) and \(T_2\), and sufficiently small
\(\varepsilon\).
\end{remark}



\end{document}